\relax
\documentclass[letterpaper]{article} 
\usepackage{aaai22}  
\usepackage{times}  
\usepackage{helvet}  
\usepackage{courier}  
\usepackage[hyphens]{url}  
\usepackage{graphicx} 
\usepackage{natbib}  
\usepackage{caption} 
\DeclareCaptionStyle{ruled}{labelfont=normalfont,labelsep=colon,strut=off} 
\usepackage{algorithm}
\usepackage{algorithmic}

\usepackage{newfloat}
\usepackage{listings}
\floatstyle{ruled}
\newfloat{listing}{tb}{lst}{}
\floatname{listing}{Listing}

\usepackage{microtype}
\usepackage{amsmath}
\usepackage{amsmath,amssymb, amsthm, amstext}
\newtheorem{theorem}{Theorem}

\usepackage{mathrsfs}
\usepackage{extarrows}
\usepackage{multirow}
\usepackage{graphicx}
\usepackage{booktabs}
\usepackage{caption}
\usepackage{subcaption}
\usepackage{enumitem}
\usepackage{arydshln}
\usepackage{xcolor}
\usepackage{svg}
\usepackage{tikz}

\makeatletter
\def\adl@drawiv#1#2#3{%
        \hskip.5\tabcolsep
        \xleaders#3{#2.5\@tempdimb #1{1}#2.5\@tempdimb}%
                #2\z@ plus1fil minus1fil\relax
        \hskip.5\tabcolsep}
\newcommand{\cdashlinelr}[1]{%
  \noalign{\vskip\aboverulesep
           \global\let\@dashdrawstore\adl@draw
           \global\let\adl@draw\adl@drawiv}
  \cdashline{#1}
  \noalign{\global\let\adl@draw\@dashdrawstore
           \vskip\belowrulesep}}
\makeatother
\title{The Neglected Sibling: Isotropic Gaussian Posterior for VAE}
\author {
    Lan Zhang\ \ \ 
    Wray Buntine\ \ \ 
    Ehsan Shareghi\ \ \ 
}
\affiliations {
    Department of Data Science and AI, Monash University\\
    Lan.Zhang@monash.edu 
}

\begin{document}

\maketitle

\begin{abstract}
Deep generative models have been widely used in several areas of NLP, and various techniques have been proposed to augment them or address their training challenges. In this paper, we propose a simple modification to Variational Autoencoders~(VAEs) by using an Isotropic Gaussian Posterior~(IGP) that allows for better utilisation of their latent representation space. This model avoids the sub-optimal behavior of VAEs related to inactive dimensions in the representation space
. We provide both theoretical analysis, and empirical evidence on various datasets and tasks that show IGP leads to consistent improvement on several quantitative and qualitative grounds, from downstream task performance and sample efficiency to robustness. 
Additionally, we offer insights about the representational properties encouraged by IGP and also show that its gain generalises to image domain as well. \footnote{For code and data: \url{https://github.com/lanzhang128/IGPVAE}}

\end{abstract}
\section{Introduction}
Variational Autoencoders (VAEs)~\citep{kingma2013} have been widely used in various areas of NLP, from representation learning for downstream tasks~\citep{li-etal-2020-optimus, ijcai2017-582}, to generation~\citep{prokhorov2019, bowman2015}, and semi-supervised learning~\cite{zhu-etal-2021-combining,DBLP:conf/acl/ChoiKL19, DBLP:conf/acl/NeubigZYH18, DBLP:conf/aaai/XuSDT17}. 
In recent years, most of the developments around VAEs have focused on avoiding the commonly known posterior collapse problem~\cite{bowman2015} which leads to learning sub-optimal representations~\citep{DBLP:journals/corr/abs-2004-14758,fu-etal-2019-cyclical,li-etal-2019-surprisingly,dieng-et-al-2019-avoiding,he2019,higgins2016,Yang2017,bowman2015}. 

Despite the success of these techniques in avoiding the collapse, a non-collapsed VAE can still utilise the representation space sub-optimally~\citep{prokhorov2019,he2019,burda2015},  
%
as very commonly the learned representations do not fully utilise the latent space to encode information. 
Additionally, vanilla VAEs induce representation spaces as soft ellipsoidal regions~\citep{bowman2015} and in such a space, the mean vectors are less representative, specifically as features for downstream tasks~\cite{bosc-vincent-2020-sequence}.
%

\begin{figure}[t]
    \centering
        \includegraphics[trim=0 0.8cm 0 0, clip, width=0.23\textwidth]{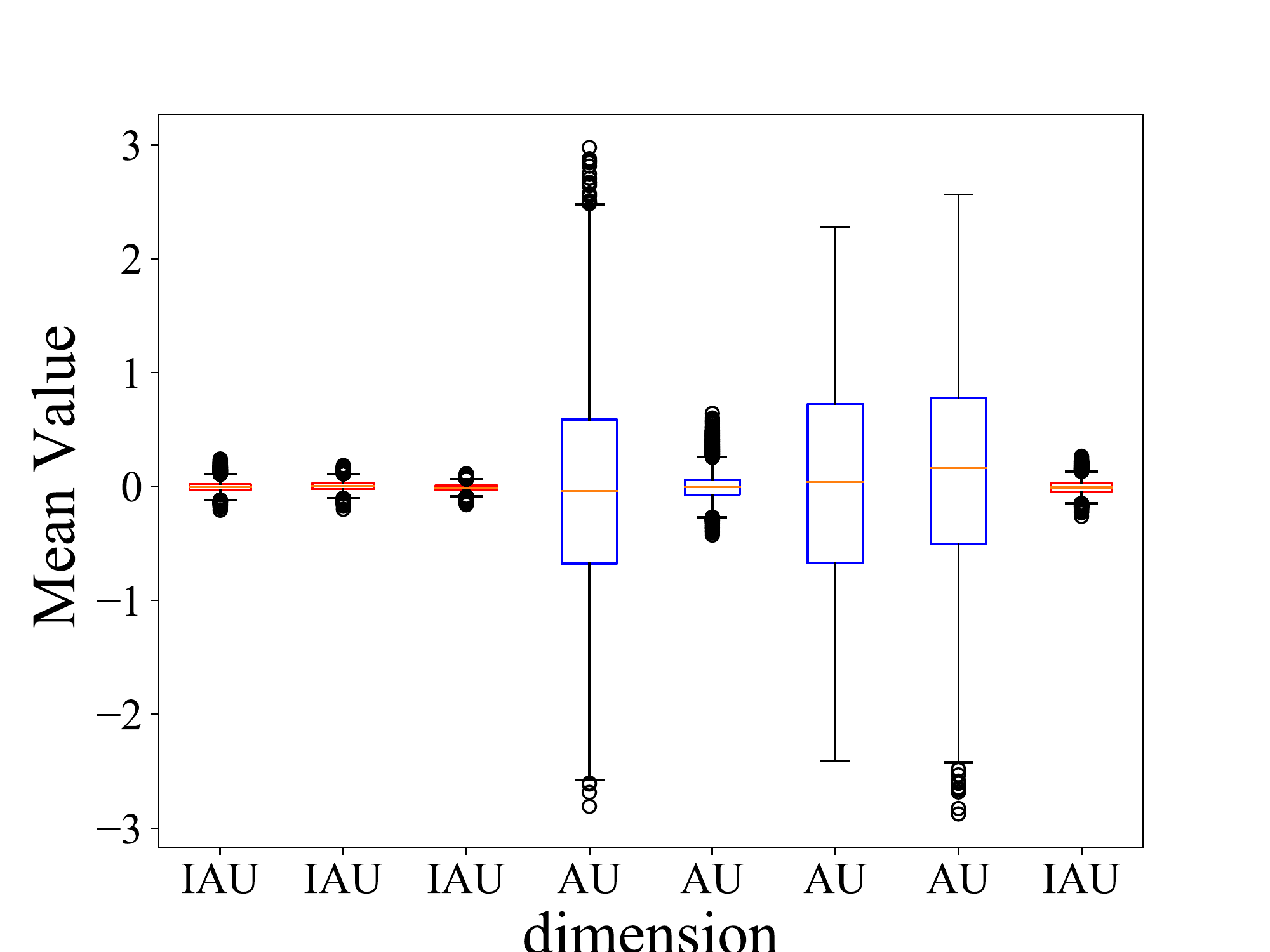}
        \includegraphics[trim=0 0.8cm 0 0, clip,width=0.23\textwidth]{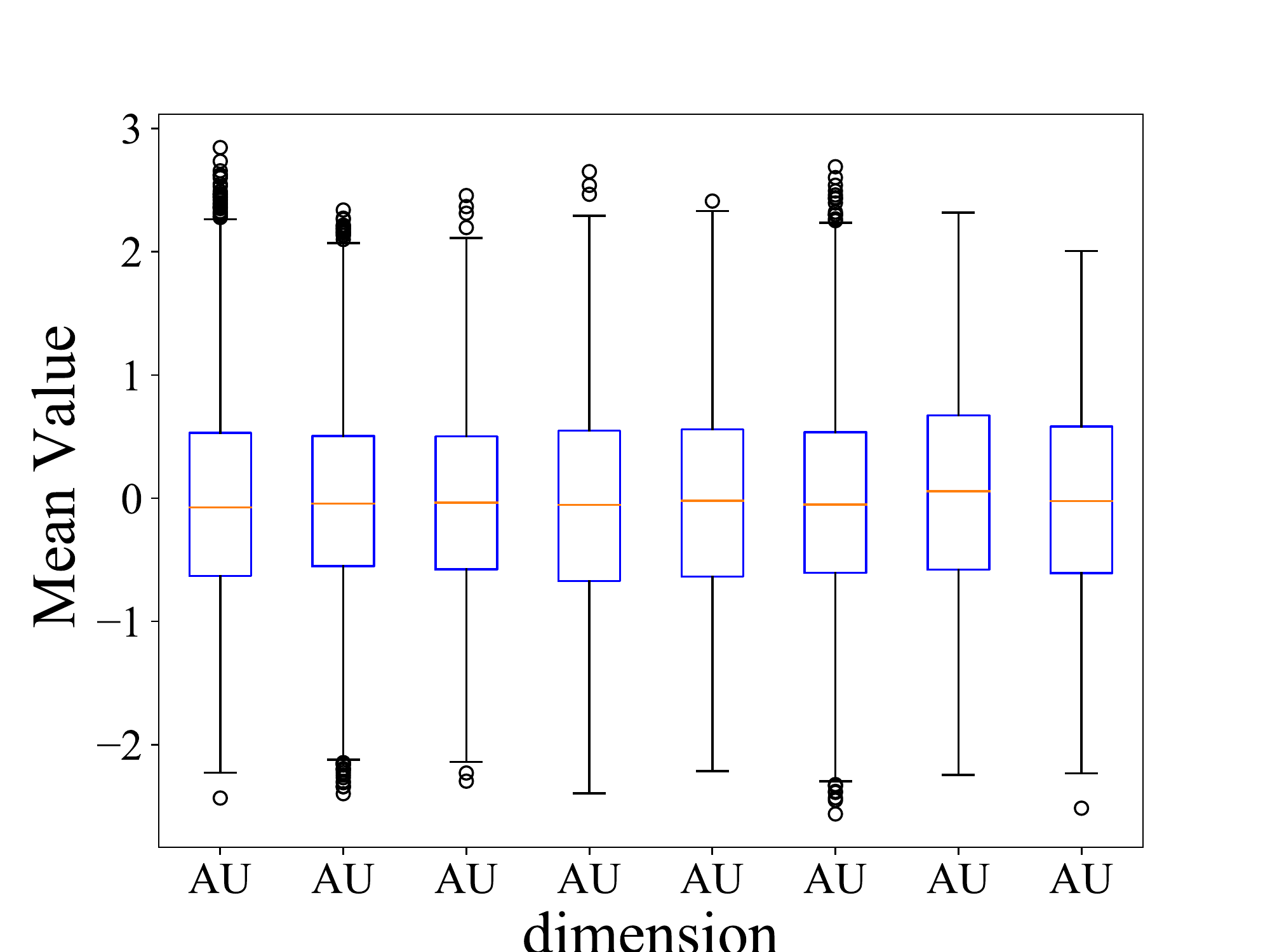}\\
        \includegraphics[width=0.23\textwidth]{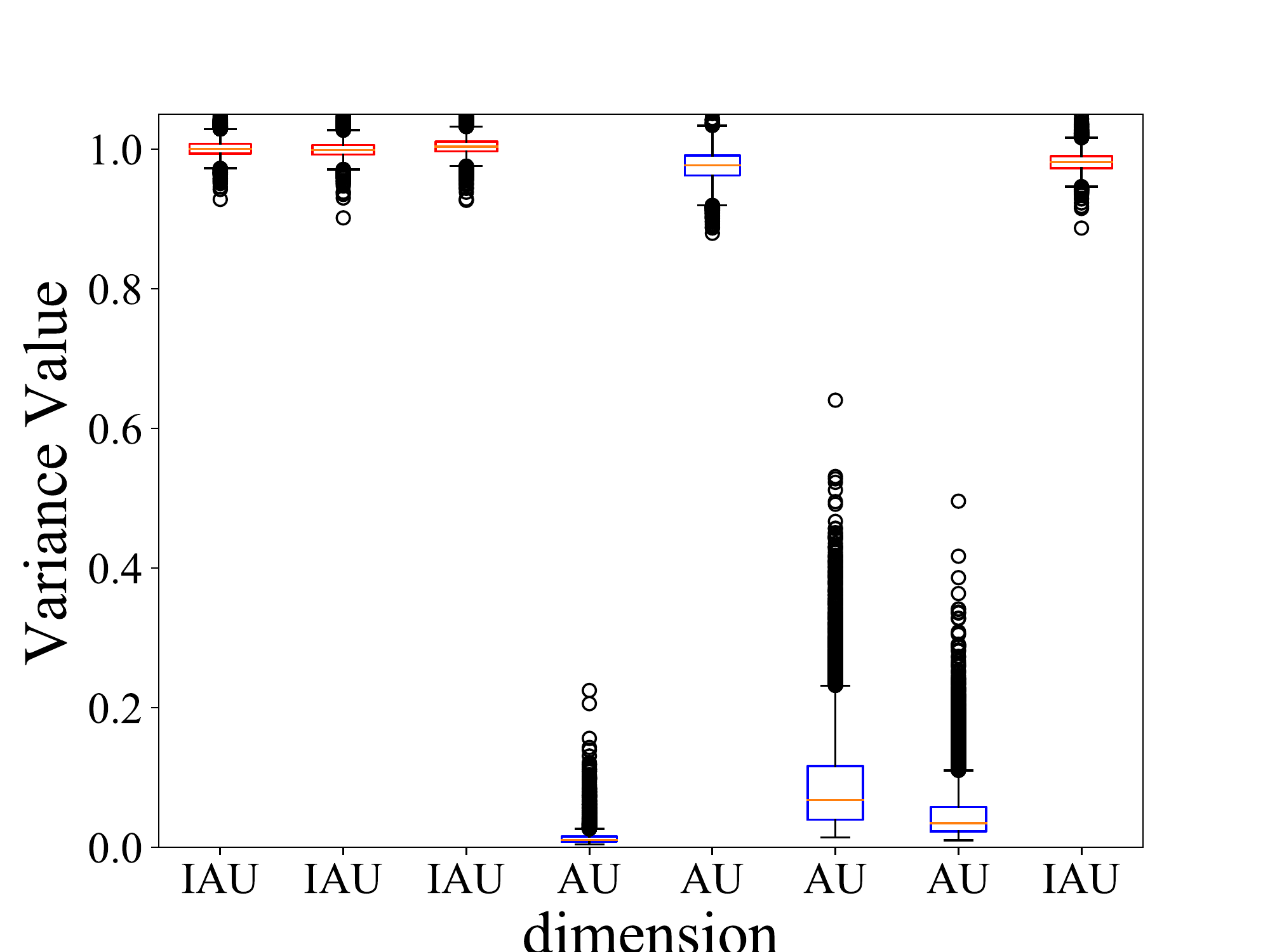}
        \includegraphics[width=0.23\textwidth]{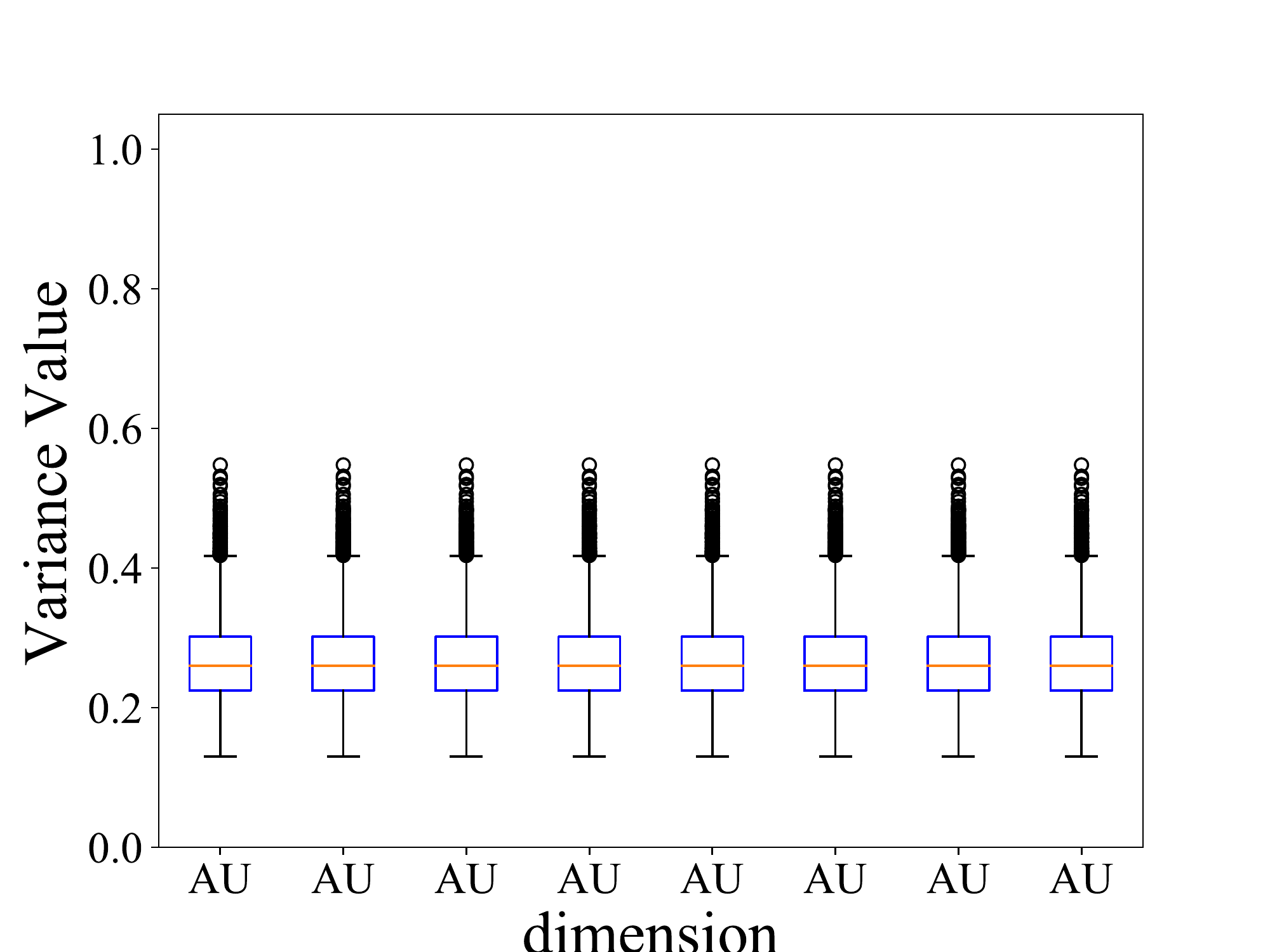}
    \caption{An observation of the latent space in VAEs trained on SNLI Corpus~\citep{bowman-etal-2015-large}. Boxplots of posteriors' mean (top) and variance (bottom) are produced on the test set. Left plot corresponds to a non-collapsed VAE (using $C=5$, Eq.~\ref{eq:mod_elbo}), right plots correspond to the same setting but with an isotropic posterior. Blue and red boxes denote results of active units (AU) and inactive units (IAU), respectively. }
    \label{fig:obser}
\end{figure}

Figure~\ref{fig:obser} (Left) reports the distribution of mean and variance for VAE (collapse is avoided using a variant of $\beta$-VAE~\cite{prokhorov2019}) with a standard Diagonal Gaussian Posterior. The empirical observation in Figure~\ref{fig:obser} suggests the uni-dimensional posterior distribution on inactive dimensions (a.k.a. Inactive Units\footnote{A dimension $u$ is defined to be active if   $A_u={\rm Cov}_{\mathbf{x}}(\mathbb{E}_{u\sim q\left(u|\mathbf{x}\right)}\left[u\right])$ is larger than 0.01, where $\rm{Cov}$ denotes covariance.} \cite{burda2015}) tend to exhibit an undesired pathological behaviour of being close to the prior distribution (e.g. variance of 1, mean of 0), meaning these dimensions do not carry any information from data. This phenomena, in the case of a collapsed VAE, corresponds to all dimensions being inactive. However, in a non-collapsed VAE, it is required from the posterior to encode information in the latent code, leading to at least a subset of  dimensions being active.

 In this paper we argue the aforementioned issues could be rectified by tying the variance of dimensions on the posterior, which would encourage encoder of VAEs towards the extremes where \emph{all} dimensions are either active or inactive. In the Gaussian case, such property could be encouraged by using an Isotropic Gaussian, corresponding to Figure~\ref{fig:obser} (Right). This only involves a simple modification of VAE by replacing its posterior with an Isotropic Gaussian distribution $\mathcal{N}(\boldsymbol{\mu}, \sigma^2\boldsymbol{I})$, which is similar to vanila VAE's posterior with the exception that all dimensions share the same unified variance. 
%
%
%

Our experimental results across text and image modalities indicate that the proposed modification consistently outperforms the vanilla VAE, and strong baseline such as IWAE~\cite{burda2015} on several quantitative and qualitative bases. We also illustrate that the proposed simple modification improves representational robustness to input perturbation, as well as sample efficiency in down stream tasks. Finally, we provide a theoretical insight about the efficiency of Isotropic VAE, as well as empirical evidence to highlight its optimisation merits.


\section{Variational Autoencoder (VAE)}\label{sec:background}
Let $\mathbf{x}$ denote datapoints in data space and $\mathbf{z}$ denote latent variables in the latent space, and assume the datapoints are generated by the combination of two random processes: The first random process is to sample a point $\mathbf{z}^{(i)}$ from the latent space in VAEs with prior distribution of $\mathbf{z}$, denoted by $p(\mathbf{z})$. The second random process is to generate a point $\mathbf{x}^{(i)}$  from the data space, denoted by $p(\mathbf{x}|\mathbf{z}^{(i)})$. VAE uses a combination of a probabilistic encoder $q_\phi(\mathbf{z}|\mathbf{x})$ and decoder $p_\theta(\mathbf{x}|\mathbf{z})$, parameterised by $\phi$ and $\theta$, to learn this statistical relationship between $\mathbf{x}$ and $\mathbf{z}$. VAE is trained by maximizing the lower bound of the logarithmic data distribution $\log p(\mathbf{x})$, called evidence lower bound (ELBO), $\mathcal{L}(\phi,\theta;\textbf{x})$:
\begin{equation}
    \label{eq:elbo}
    \mathbb{E}_{q_\phi(\mathbf{z}|\mathbf{x})}[\log(p_\theta(\mathbf{x}|\mathbf{z}))]-\textrm{KL}(q_\phi(\mathbf{z}|\mathbf{x})||p(\mathbf{z}))
\end{equation}
The first term of objective function is the expectation of the logarithm of data likelihood under the posterior distribution of $z$. The second term is KL-divergence, measuring the distance between the recognition distribution $q_\phi(\mathbf{z}|\mathbf{x})$ and the prior distribution $p(\mathbf{z})$ and can be seen as a regularisation.

\subsection{Posterior Collapse}\label{sec:nocollapse}
In the presence of auto-regressive and powerful decoders, a common optimisation challenge of training VAEs in text modelling is called posterior collapse, where the learned posterior distribution $q_\phi(\mathbf{z}|\mathbf{x})$, collapses to the prior  $p(\mathbf{z})$. Posterior collapse results in
the latent variables $\mathbf{z}$ being ignored by the decoder. Several strategies have been proposed to alleviate this problem from different angles such as choice of decoders~\citep{Yang2017,bowman2015}, adding more dependency between encoder and decoder~\citep{dieng-et-al-2019-avoiding}, adjusting the training process or objective function \citep{DBLP:journals/corr/abs-2004-14758,fu-etal-2019-cyclical,he2019,bowman2015}, and imposing direct constraints to the KL term~\citep{pelsmaeker-aziz-2020-effective,razavi2019,higgins2016}.
%
In this work, we apply constraint optimisation~\citep{burgess2018,prokhorov2019}:
\begin{multline}
\setlength{\abovedisplayskip}{3pt}
\setlength{\belowdisplayskip}{3pt}
\setlength{\abovedisplayshortskip}{3pt}
\setlength{\belowdisplayshortskip}{3pt}
    \label{eq:mod_elbo}
    \mathcal{L}(\phi,\theta;\textbf{x})=\mathbb{E}_{q_\phi(\mathbf{z}|\mathbf{x})}[\log(p_\theta(\mathbf{x}|\mathbf{z}))]-\\\beta\left|{\rm KL}(q_\phi(\mathbf{z}|\mathbf{x})||p(\mathbf{z}))-C\right|
\end{multline}
where $C$ is a positive real value which represents the target KL-divergence term value. We set $\beta=1$ to make sure the weights of the two terms balance, noting that it acts as a Lagrange Multiplier \citep{boyd2004convex}. This also has an information-theoretic interpretation, where the placed constraint $C$ on the KL term is seen as the amount of information transmitted from a sender (encoder) to a receiver (decoder) via the message~($\mathbf{z}$)~\cite{alemi2018fixing}.
\section{Isotropic Gaussian Posterior (IGP)}\label{sec:method}
A common behaviour of VAEs is the presence of inactive representation units across the entire dataset, causing the number of utilised dimensions to be even far smaller than the number of potential generative factors behind any real-world dataset. The soft ellipsoidal representation space of VAEs is known to lead to less representative mean vectors~\cite{bosc-vincent-2020-sequence}. 

We argue that tying the variance of dimensions on the posterior will avoid the aforementioned issue since the encoder of VAEs would be forced to either use all dimensions or none and the learned latent space is soft spherical. In the Gaussian case, this corresponds to using an Isotropic Gaussian, a subclass of diagonal Gaussian distribution $\left\{\mathcal{N}(\boldsymbol{\mu},\sigma^2\boldsymbol{I})|\boldsymbol{\mu}\in{\mathbf{R}^{n}},\sigma\in\mathbf{R}^+\right\}$, as the posterior. Although an AE can have the same effect, tying variances to a constant value 0 rather than a learned value makes the regularizer no use and the generation quality worse.


Tying the variances in IGP could impose a different pathological pattern on units' activation as illustrated in the right panels of Figure~\ref{fig:obser}. As demonstrated the use of IGP in a non-collapsed VAE encourages the number of active units to be pushed towards the maximum (i.e., matching the representation dimensionality). These observations generalise to other datasets and configurations of $C$. 

Additionally, the use of IGP allows the estimation of variance more accurately. Suppose we have $N$ samples with the same posterior. For a K-dimension diagonal Gaussian posterior, we will have an estimate of variance with standard deviation approximately ${\hat\sigma}_k^2\sqrt{\frac{2}{N}}$ for each dimension $k$, whereas for an isotropic Gaussian posterior, we will have a unified estimation of variance with standard deviation approximately ${\hat\sigma}^2\sqrt{ \frac{2}{NK}}$, 
where ${\hat\sigma}_k^2$ and ${\hat\sigma}^2$ denote the estimates of the variance. 
Moreover, with $K$ different ${\hat\sigma}_k^2$ estimates, a few may differ
substantially from their best values by chance.
\section{Experiments}\label{sec:results}
\begin{figure*}[t]
    \centering
    \begin{subfigure}[b]{0.48\textwidth}
        \centering
        \includegraphics[trim=0 8mm 9.5cm 0, clip,scale=0.6]{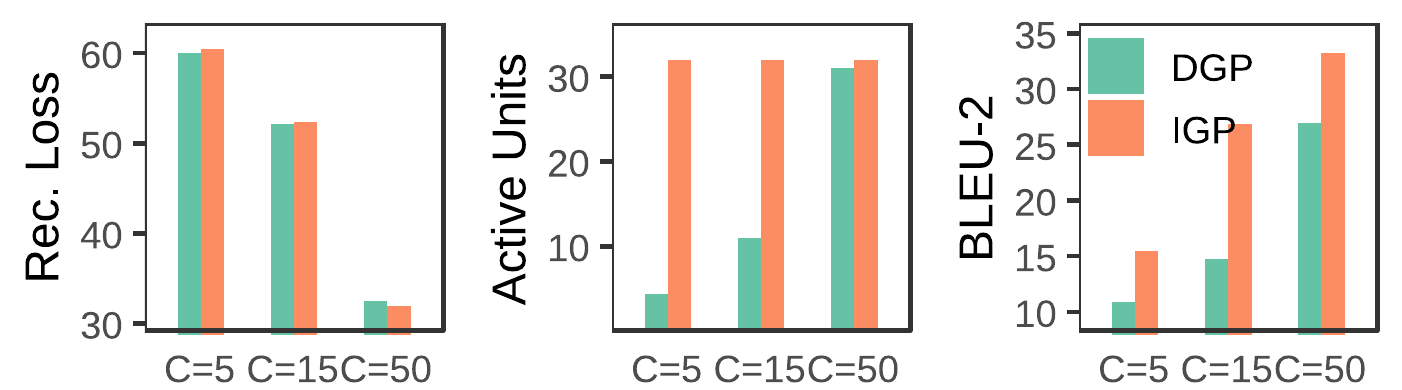}
        \caption{CBT Corpus~\citep{HillBCW15} }
        \label{fig:cbt_result}
    \end{subfigure}\hfill
    \begin{subfigure}[b]{0.48\textwidth}
        \centering
        \includegraphics[trim=0 8mm 9.5cm 0, clip,scale=0.6]{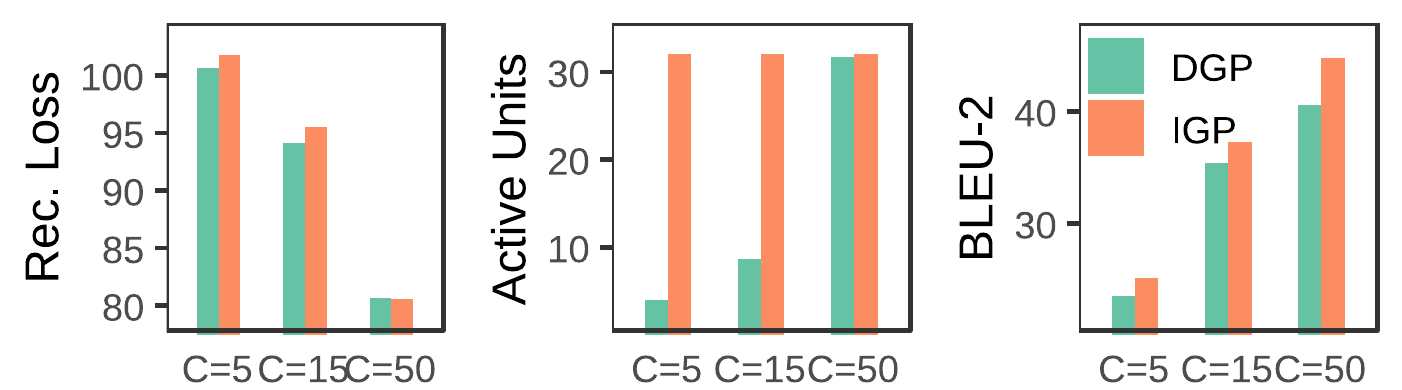}
        \caption{DBpedia Corpus~\citep{10.5555/2969239.2969312}}
        \label{fig:dbpedia_result}
    \end{subfigure}
    \begin{subfigure}[b]{0.48\textwidth}
        \centering
        \includegraphics[trim=0 8mm 9.5cm 0, clip,scale=0.6]{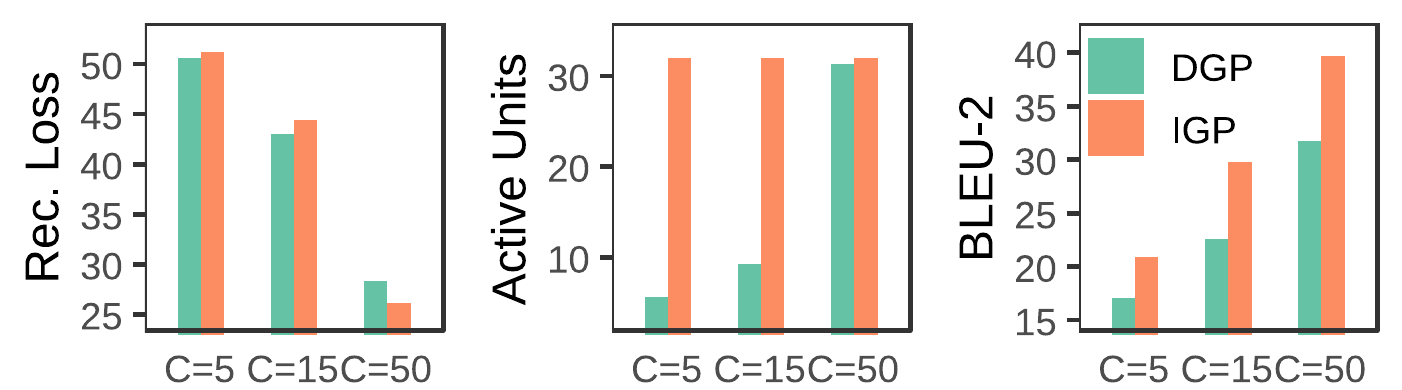}
        \caption{Yahoo Question Corpus~\citep{10.5555/2969239.2969312}}
        \label{fig:yahoo_result}
    \end{subfigure}\hfill
    \begin{subfigure}[b]{0.48\textwidth}
        \centering
        \includegraphics[trim=0 8mm 9.5cm 0, clip,scale=0.6]{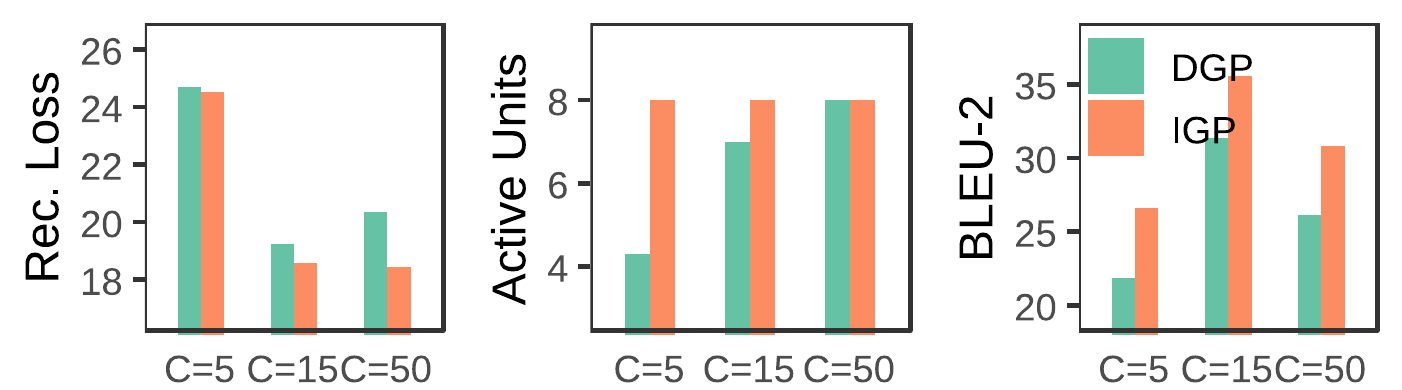}
        \caption{SNLI Hypothesis Corpus~\citep{bowman-etal-2015-large}}
        \label{fig:snli_result}
    \end{subfigure}
    \caption{Results are calculated on the test set (average of 3 runs reported). Rec. Loss, AU, DGP, and IGP denote reconstruction loss, number of Active Units, vanilla VAE, and VAE with isotropic Gaussian posteriors, respectively. AU is bounded by the dimensionality of $z$ (32 for CBT, DBpedia, Yahoo Question, and 8 for SNLI). KL for both IGP and DGP roughly matches the set target via $C$. For full results see \emph{Supplementary Material}.}
    \label{fig:loss}
\end{figure*}

We trained our models on four datasets, CBT~\citep{HillBCW15}, Yahoo Question and DBpedia~\citep{10.5555/2969239.2969312}, and SNLI Hypothesis \citep{bowman-etal-2015-large}. See Table~\ref{tab:data} for data statistics. We compare the choice of isotropic Gaussian posterior (IGP) with vanilla diagonal Gaussian posterior (DGP) on various quantitative and qualitative settings, from reconstruction and representation utilisation (\S\ref{sec:basic}) to downstream classification task and sample efficiency (\S\ref{sec:classification}). We also provide additional analysis about the representational properties encouraged by IGP, showing it works well with other priors, and could also be used as a warm initialiser to boost performance of vanilla VAEs (\S\ref{sec:analysis}).

\paragraph{Model Configurations.} We use the VAE architecture of \citep{bowman2015} and concatenate the latent code with word embedding at every timestamp as the input of the decoder. For VAE with IGP, we just produce one variance value and assign it to be the variance of posterior for all dimensions. At decoding phase, we use greedy decoding. The dimensions for word embedding, encoder-decoder LSTMs, and latent code are (200, 512, 32) for CBT, Yahoo and DBpedia, and (50, 128, 8) for SNLI. Three different values of $C$ (see Eq.~\ref{eq:mod_elbo}) is used on each dataset to explore the impact of the amount of information transmitted by the code. We also adopt Autoencoder (AE) as a baseline. \footnote{We also try to incorporate Importance Weighted Autoencoder (IWAE;\citet{burda2015}) as another baseline. However, this model still has the KL-collapse issue, which makes it unsuitable for comparison in this setting.} All models are trained from 3 random starts for 20 epochs and 128 batch size using Adam~\citep{kingma2014} with learning rate 0.0005.

\subsection{Basic Findings}\label{sec:basic}
Figure~\ref{fig:loss} reports the reconstruction loss, active units (AU; \citealp{burda2015}) and BLEU-2~\citep{bleu} for 3 settings of $C=5,15,50$ on 4 datasets. KL in all cases match the set target C. For full results, including KL, BLEU-4 and ROUGE-2/4~\citep{lin-2004-rouge}, see \emph{Supplementary Material}. 

\paragraph{Target $C$.} The usage of $C$ can effectively control the KL-divergence to a certain level. The reconstruction loss generally drops with the increase of $C$ value. We observe the  same pattern, comparing the loss for VAE (DGP) and VAE with IGP under each $C$ value. While the increase of $C$ is interpreted as the amount of information transmitted through the representation~\citep{alemi2018fixing,prokhorov2019}, the use of IGP allows for a better utilisation of the latent space and consequently a more accurate estimation of the posterior. 
\begin{table}[t]
    \centering
    \small
    \begin{tabular}{l c c c c c}
        \toprule
        Data & Train & Dev & Test & Vocab & \# Classes \\
        \hline
        CBT & 192K & 10K & 12K & 12K & -\\
        DBpedia & 140K & 14K & 14K & 12K & 14\\
        Yahoo Question & 100K & 10K & 10K & 12K & 10\\
        SNLI & 100K & 10K & 10K & 7K & -\\
        \bottomrule
    \end{tabular}
    \caption{Statistics of datasets.}
    \label{tab:data}
\end{table}
\noindent\textbf{Active unit (AU).} AU measures the activity of dimensions in the representation space. Intuitively, if the changes of the mean value of posterior distribution on a dimension across all data points in the training set are significant, the corresponding dimension is active and carries information. As can be observed, compared to the diagonal Gaussian counterpart, IGP can activate significantly more dimensions, almost all dimensions in many cases.

\noindent\textbf{BLEU-2.}
VAEs with IGP significantly and consistently outperform vanilla VAEs. Having a comparable BLEU-2 on CBT dataset for IGP ($C=5$) vs. DGP ($C=15$), highlights the potential of IGP even at a lower $C$ value: 15.45 vs 14.76.

\noindent\textbf{Does IGP work only with $C$?}
We chose setting a target KL via $C$ as our strategy to avoid posterior collapse. While this strategy offers an information-theoretic interpretation of VAE behaviours, we still want to examine if our proposed IGP could work with another commonly used method, $\beta$-VAE~\citep{higgins2016} which incorporates a $\beta$ coefficient for the KL term. We train VAEs on SNLI with the same training configuration as before and three different $\beta$. The results are reported in Table~\ref{tab:beta}. Although non-comparable KL values makes it difficult to fairly compare models, we still observe IGP VAEs can activate more dimensions, and have better reconstruction quality with a rather smaller KL-divergence than DGP VAEs (e.g., IGP VAEs ($\beta=0.4$) and DGP VAEs ($\beta=0.2$)).
\begin{table}[t]
    \centering
    \begin{tabular}{l c c c c}
        \toprule
        Model & Rec. & KL & AU & BLEU-2/4\\
        \toprule
        $\beta=0.8$, DGP & 28.14 & 0.75 & 2.0 & 12.74/3.33\\
        $\beta=0.8$, IGP & 28.52 & 0.27 & 7.3 & 11.18/2.92\\
        \cdashlinelr{1-5}
        $\beta=0.4$, DGP & 22.72 & 7.66 & 5.7 & 22.43/9.90\\
        $\beta=0.4$, IGP & 20.84 & 10.28 & 8.0 & 31.48/17.61\\
        \cdashlinelr{1-5}
        $\beta=0.2$, DGP & 18.93 & 14.71 & 7.7 & 30.53/18.24\\
        $\beta=0.2$, IGP & 16.38 & 19.04 & 8.0 & 36.72/24.87\\
        \bottomrule
    \end{tabular}
    \caption{$\beta$-VAE results of SNLI test set. We report mean across 3 runs for space reason.}
    \label{tab:beta}
\end{table}
\subsection{Classification}\label{sec:classification}
We froze the encoder from trained VAEs and AEs, and place a classifier on top to use the mean vector representations from the encoder as a feature to train a classifier. We split representations into training and test set with proportion 80\% and 20\%, and used  20\% of training data as validation set. For the classifier, we used a 2-hidden-layer MLP with 128 neurons and ReLU activation function at each layer. Classifiers were optimised with Adam (learning rate 0.001) for 20 epochs. We trained 10 randomly initialised classifiers and use the mean of classification accuracy as the final accuracy of one VAE or AE encoder. We report the mean and standard deviation of AEs and VAEs with three settings of $C$ in Figure~\ref{fig:clf} (top).

\paragraph{Results.} Overall, the representations of most VAEs with IGP lead to a significant improvement of classification accuracy compared to vanilla VAEs. In the only exception (i.e, $C=5$ on DBpedia), two models have comparable results with no model having any statistically significant advantage. We attribute this success to having a more representative mean which is encouraged by IGP. One notable thing is that VAEs with diagonal Gaussian posterior do not perform as good as AEs regardless of $C$ choice, whereas IGP VAEs ($C=15,50$) achieve similar and better classification accuracy on DBpedia and Yahoo Question, respectively. 
\begin{figure}[t]
    \centering
    \includegraphics[scale=0.43]{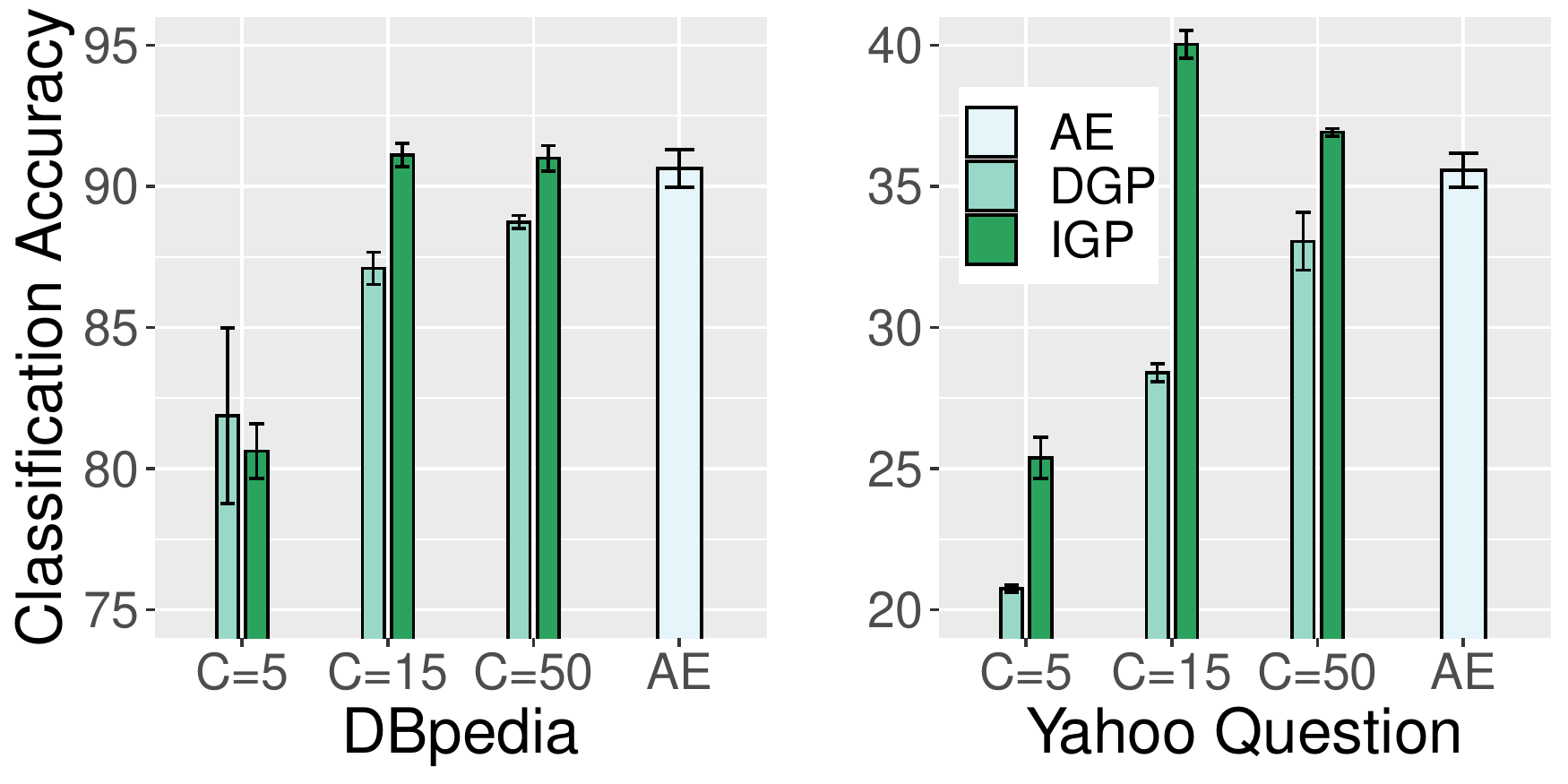}
    \includegraphics[scale=0.43]{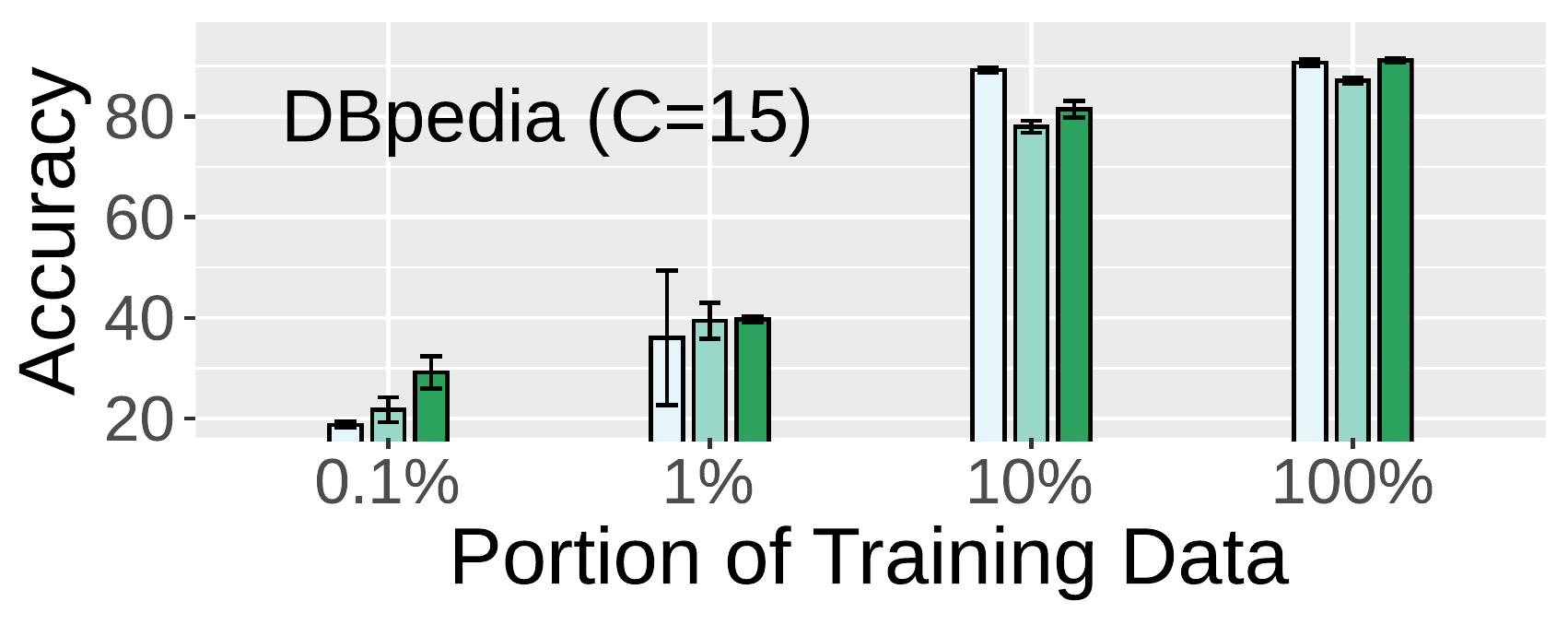}
    \caption{Classification accuracy on DBpedia (top-left) and Yahoo (top-left) with and without the isotropic Gaussian posterior (IGP) under different $C$ values. Also, classification accuracy for $C=15$ trained on various portion of DBpedia (bottom). Results are reported as mean and std across 3 VAE encoders.}
    \label{fig:clf}
\end{figure}
\begin{figure}
    \centering
    \includegraphics[trim=10mm 0 3mm 0, clip,width=0.24\textwidth]{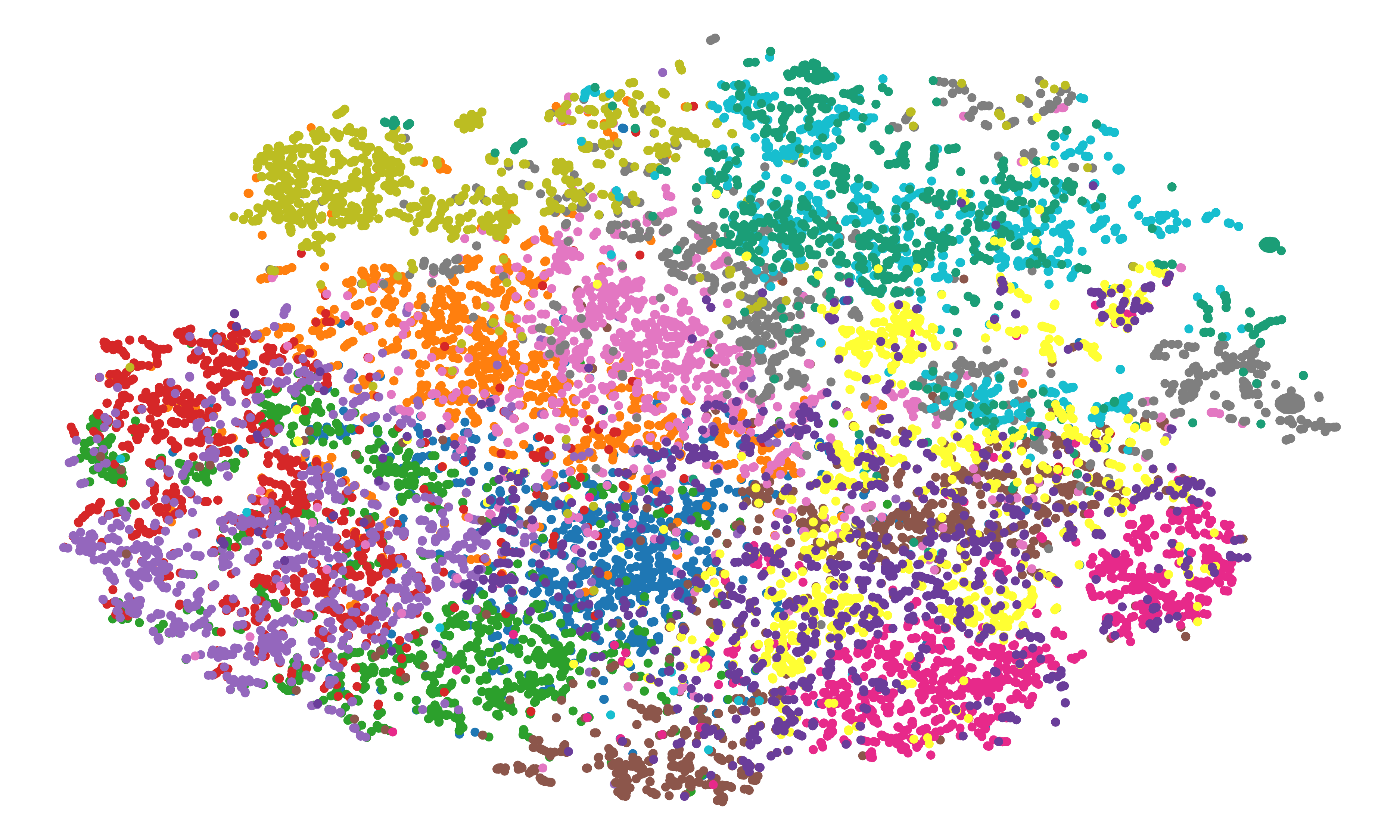}
    \hskip -2ex
    \includegraphics[trim=10mm 0 3mm 0, clip,width=0.24\textwidth]{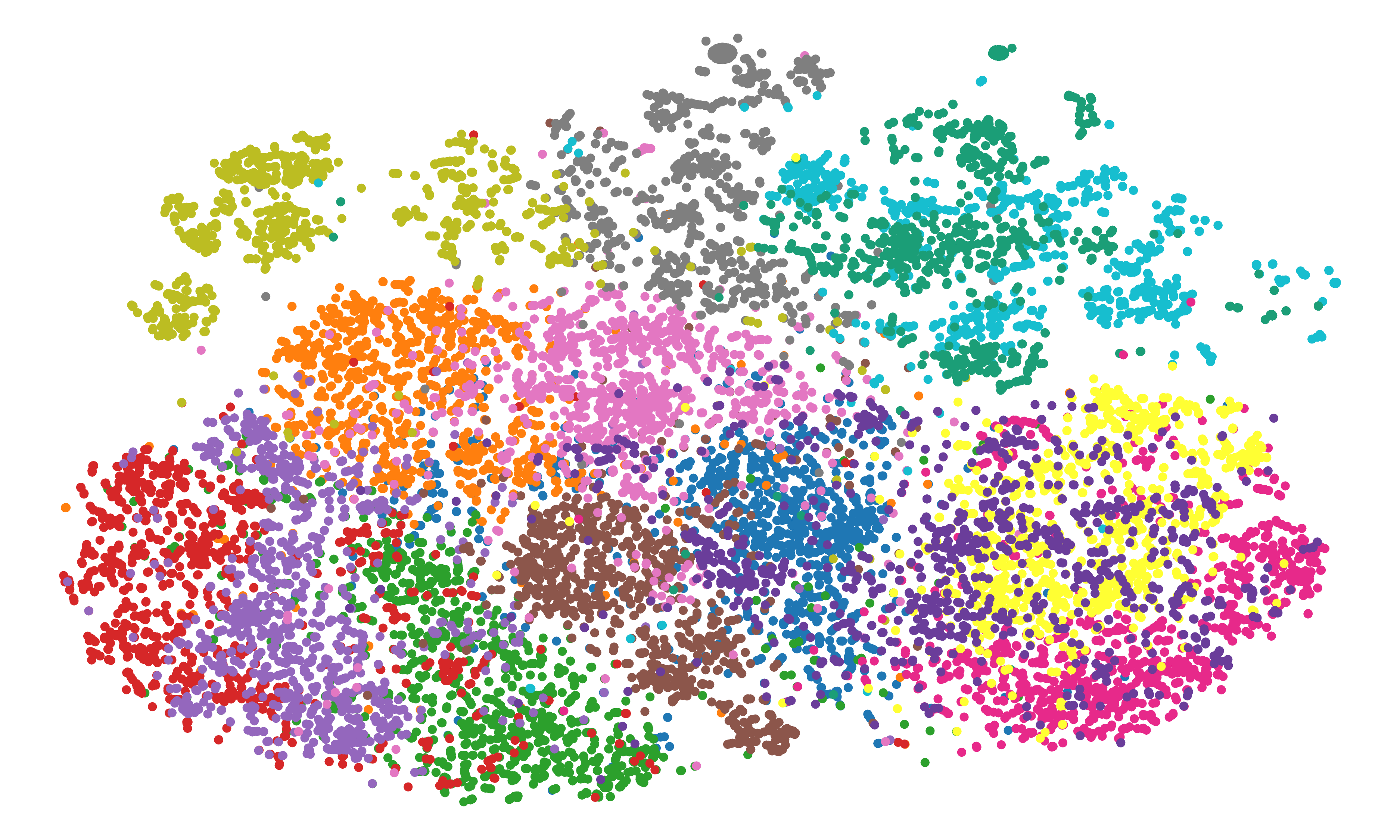}
    \includegraphics[trim=5.3cm 16.2cm 0.6cm 0.6cm, clip,scale=0.18]{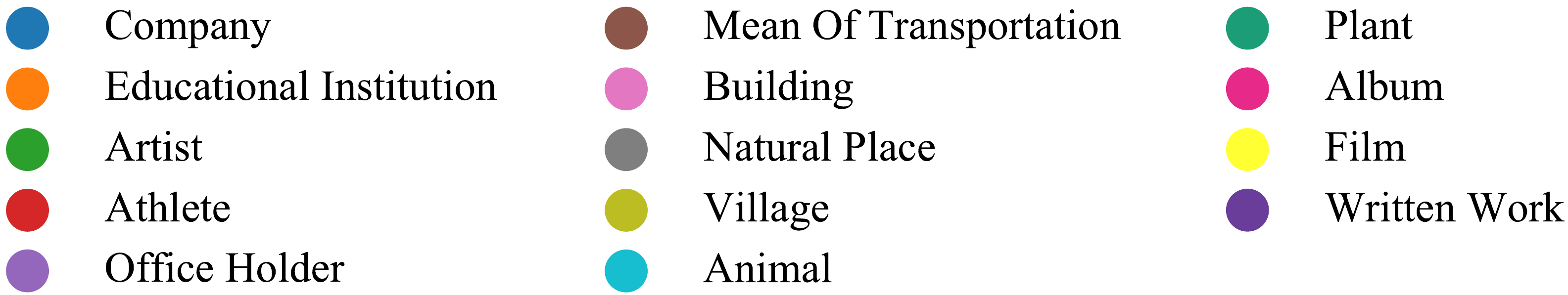}
    \caption{The t-sne plots of sentence representations of DBpedia test set for $C=15$. \textbf{Left:} DGP; \textbf{Right:} IGP.}
    \label{fig:tsne}
\end{figure}

\paragraph{Visualization of the Representation Space.} We use t-sne~\cite{JMLR:v9:vandermaaten08a} to visualize the learned representation space of VAE and VAE with IGP (both with $C=15$) for DBpedia (14 classes) in Figure~\ref{fig:tsne}. As illustrated in the right plot, the clusters of classes are more concentrated in the space of IGP VAE with less overlap among classes. For instance, in the representation space of vanilla VAE (left plot), Artist, Athlete, Office Holder classes are entangled, whereas in IGP VAE, representations of these three classes show a clear clustering pattern (compare the green, red, purple points at the bottom-left of both t-sne plots). The better separation of classes in the representation space (achieved fully unsupervised) explains the superior performance of IGP when further trained in the classification task.

\paragraph{Few-Shot Setting.} We adopt few-shot setting to compare sample efficiency of both VAEs, by using 0.1\%, 1\% and 10\% of training data of DBpedia to them with $C=15$ and do classification on the test set as before. Accuracy scores are reported in Figure~\ref{fig:clf} (bottom). Using VAE with IGP in a few-shot setting can lead to a  better performance on this classification task. For instance, the mean accuracy gap at 0.1\% is quite significant being above 7 points, and VAE  gets the gap down to 4 points at 100\% (still significant). 
\begin{table}[t]
    \centering
    \begin{tabular}{l l c c}
        \toprule
        & Model & \textbf{DBpedia} & \textbf{Yahoo Question}\\
        \hline
        \textbf{
          \parbox[t]{3mm}{\multirow{8}{*}{\rotatebox[origin=c]{90}{Agreement}}}}
        & AE & $\text{91.85}_{\pm\text{0.68}}$ & $64.77_{\pm0.59}$\\
        \cdashlinelr{2-4}
        & $C=5$, DGP & $53.79_{\pm3.83}$ & $24.33_{\pm0.72}$\\
        & $C=5$, IGP & $46.62_{\pm0.51}$ & $28.75_{\pm0.48}$\\
        \cdashlinelr{2-4}
        & $C=15$, DGP & $80.45_{\pm1.71}$ & $35.89_{\pm0.95}$\\
        & $C=15$, IGP & $84.24_{\pm0.74}$ & $54.42_{\pm0.21}$\\
        \cdashlinelr{2-4}
        & $C=50$, DGP & $89.40_{\pm0.37}$ & $51.09_{\pm0.33}$\\
        & $C=50$, IGP & $\mathbf{93.48_{\pm0.37}}$ & $\mathbf{64.88_{\pm0.37}}$\\
        \hline
        \textbf{
          \parbox[t]{3mm}{\multirow{8}{*}{\rotatebox[origin=c]{90}{Robustness}}}}
        & AE & $83.09_{\pm0.81}$ & $30.68_{\pm0.32}$\\
        \cdashlinelr{2-4}
        & $C=5$, DGP & $68.77_{\pm4.76}$ & $18.47_{\pm0.50}$\\
        & $C=5$, IGP & $72.46_{\pm1.00}$ & $22.28_{\pm0.30}$\\
        \cdashlinelr{2-4}
        & $C=15$, DGP & $76.42_{\pm1.18}$ & $24.08_{\pm0.28}$\\
        & $C=15$, IGP & $\mathbf{83.49_{\pm0.63}}$ & $\mathbf{33.99_{\pm0.20}}$\\
        \cdashlinelr{2-4}
        & $C=50$, DGP & $79.18_{\pm0.38}$ & $28.29_{\pm1.02}$\\
        & $C=50$, IGP & $83.04_{\pm0.70}$ & $32.16_{\pm0.19}$\\
        
        \bottomrule
    \end{tabular}
    \caption{Robustness and Proportion of reserved label information on test sets.}
    \label{tab:agreement}
\end{table}

\paragraph{Reserved Label Information.} We conduct the agreement experiment \citep{bosc-vincent-2020-sequence}, extrinsically examining the amount of label information preserved during the reconstruction phase. We first train a regular text classifier with a 512D LSTM and two hidden MLP layers with 128 neurons and ReLU activation function each on top on the training set of DBpedia and Yahoo using the same training set as VAEs. We then calculate the macro-averaged F1-scores of reconstructed test sets and divide them with the the macro-averaged F1-score of the original test set to represent the proportion of label information reserved in reconstruction. The results are provided in Table~\ref{tab:agreement}. 

The reconstruction of IGP variant generally has kept more label information than their DGP counterparts. Surprisingly, for $C=5$, vanilla VAEs can have more label information, which shows metrics measuring reconstruction quality (e.g., measured by BLEU) do not quantify the reserved information. For full results, see \emph{Supplementary Material}.

\paragraph{Robustness.} We further apply word dropout on sentences by randomly deleting 30\% of words in a sentence, and repeating the classification experiment. The classification accuracy on the representation of polluted sentences are reported in Table~\ref{tab:agreement}. IGP VAEs outperform DGP VAEs in all cases with up-to 9\% absolute gain. Moreover, with $C\geq 15$, IGP VAEs are more robust to polluted sentences than AEs. This further suggests the superiority of IGP VAEs on classification tasks.

\begin{table}[t]
    \centering
    \small
    \begin{tabular}{ll c c}
        \toprule
        & Model & \textbf{Forward} & \textbf{Reverse}\\
        \hline
        \textbf{
        \parbox[t]{3mm}{\multirow{10}{*}{\rotatebox[origin=c]{90}{CBT}}}} 
        & Real & - & 71.88\\\cdashlinelr{2-4}
        & AE & $308.12_{\pm7.99}$ & $254.52_{\pm10.06}$\\\cdashlinelr{2-4}
        & $C=5$, DGP & $8.37_{\pm0.13}$ & $4795.16_{\pm473.99}$\\
        & $C=5$, IGP & $8.29_{\pm0.11}$ & $2117.98_{\pm164.36}$\\\cdashlinelr{2-4}
        & $C=15$, DGP & $15.06_{\pm0.48}$ & $730.28_{\pm41.52}$\\
        & $C=15$, IGP & $14.08_{\pm0.07}$ & $498.67_{\pm5.16}$\\\cdashlinelr{2-4}
        & $C=50$, DGP & $51.25_{\pm0.07}$ & $254.65_{\pm20.09}$\\
        & $C=50$, IGP & $52.20_{\pm0.81}$ & $191.39_{\pm6.70}$\\
        \hline
        \textbf{
        \parbox[t]{3mm}{\multirow{10}{*}{\rotatebox[origin=c]{90}{DBpedia}}}}
        & Real & - & 35.22\\\cdashlinelr{2-4}
        & AE & $95.95_{\pm16.60}$ & $139.80_{\pm3.97}$\\\cdashlinelr{2-4}
        & $C=5$, DGP & $3.60_{\pm0.10}$ & $4524.95_{\pm396.03}$\\
        & $C=5$, IGP & $2.74_{\pm0.03}$ & $4365.36_{\pm319.29}$\\\cdashlinelr{2-4}
        & $C=15$, DGP & $5.86_{\pm0.15}$ & $576.93_{\pm53.78}$\\
        & $C=15$, IGP & $4.47_{\pm0.03}$ & $767.13_{\pm51.00}$\\\cdashlinelr{2-4}
        & $C=50$, DGP & $15.31_{\pm0.20}$ & $158.92_{\pm9.47}$\\
        & $C=50$, IGP & $14.85_{\pm0.09}$ & $146.44_{\pm3.13}$\\
        \hline
        \textbf{
        \parbox[t]{3mm}{\multirow{10}{*}{\rotatebox[origin=c]{90}{Yahoo Question}}}} 
        & Real & - & 113.08\\\cdashlinelr{2-4}
        & AE & $418.67_{\pm58.55}$ & $605.59_{\pm78.57}$\\\cdashlinelr{2-4}
        & $C=5$, DGP & $6.96_{\pm0.27}$ & $8984.25_{\pm1591.50}$\\
        & $C=5$, IGP & $6.68_{\pm0.13}$ & $5317.33_{\pm279.97}$\\\cdashlinelr{2-4}
        & $C=15$, DGP & $17.65_{\pm0.46}$ & $1115.01_{\pm71.37}$\\
        & $C=15$, IGP & $13.23_{\pm0.18}$ & $1082.78_{\pm15.91}$\\\cdashlinelr{2-4}
        & $C=50$, DGP & $66.96_{\pm1.27}$ & $455.84_{\pm11.40}$\\
        & $C=50$, IGP & $76.95_{\pm2.11}$ & $312.53_{\pm9.30}$\\
        \bottomrule
    \end{tabular}
    \caption{Forward and Reverse perplexities.}
    \label{tab:perplexity}
\end{table}
\begin{table*}[t]
\small
    \centering
    \begin{tabular}{lp{0.88\textwidth}}
        \toprule
        \textsc{Original} &the carnegie library in unk washington is a building from 1911 . it was listed on the national register of historic places in 1982 .\\\cdashline{2-2}
        \textsc{Imputed}&the carnegie library in unk washington $\cdots$\\\cdashline{2-2}
        \textsc{DGP}&the carnegie library in unk washington is a unk ( unk ft ) high school in the unk district of unk in the province of unk in the unk province of armenia . \\\cdashline{2-2}
        \textsc{IGP}&the carnegie library in unk washington was built in 1909 . it was listed on the national register of historic places in unk was designed by architect john unk .\\\bottomrule
        \textsc{Original} &st. marys catholic high school is a private roman catholic high school in phoenix arizona . it is located in the roman catholic diocese of phoenix .\\\cdashline{2-2}
        \textsc{Imputed}&st. marys catholic high school is $\cdots$\\\cdashline{2-2}
        \textsc{DGP}&st. marys catholic high school is a unk - unk school in unk unk county new jersey united states . the school is part of the unk independent school district . \\\cdashline{2-2}
        \textsc{IGP}&st. marys catholic high school is a private roman catholic high school in unk california . it is located in the roman catholic diocese of unk . \\
        \bottomrule
    \end{tabular}
    \caption{Word imputation experiment on DBpedia test set.}
    \label{tab:rec}
\end{table*}
\subsection{Generation}\label{sec:generation}
We quantitatively evaluate the quality of generated sentences via forward perplexity and reverse perplexity experiments~\citep{pmlr-v80-zhao18b}. We first generate 100K sentences using standard Gaussian distribution. Then, we use a LSTM layer with 512 dimensions and a 200-dimension word embedding layer to form a common language model. We train the language model on the original training set and calculate standard forward perplexity on generated sentences of models. The reverse perplexity is calculated on the test set based on the language model trained on generated sentences. All number of training epochs is set to 10 and results are provided in Table~\ref{tab:perplexity}. We also include the standard perplexity on test set.

In most cases, IGP VAEs have both lower forward and reverse perplexity than DGP VAEs, which indicates their better generation quality. For the forward perplexity, both IGP and DGP VAEs have a better quality than AEs. This validates the superiority of VAE over AE in generation. However, for the reverse perplexity, VAEs are not necessarily better than AEs and those with lower $C$ value usually have significantly larger perplexity. The explanation might be that the limited capacity indicated by lower $C$ prevent VAEs to learn richer features for modelling sentences in the latent space.

\noindent\textbf{Imputing Missing Words.} We impute \%75 of words of a sentence from the test set of DBpedia, feed it to the VAE encoder and reconstruct sentence from the latent code to compare the ability of sentence completion for IGP and DGP models in~Table~\ref{tab:rec}. IGP VAE can recover more semantic and syntactic information of original sentences from the imputed sentences. In both cases, IGP VAE successfully recovers the type of the mentioned object and complete the imputed sentence with a similar structure, whereas DGP VAE fails to do so.

\begin{table}[t]
\centering
    \begin{tabular}{l l *{4}{c}}
        \toprule
        & Model & Rec. & KL & AU & CLS\\
        \hline
        & AE & 211.10 & - & 32 & 81.89\%\\
        \hline
        \textbf{
        \parbox[t]{2.5mm}{\multirow{4}{*}{\rotatebox[origin=c]{90}{DGP}}}} 
        & VAE & 227.14 & 13.33 & 8 & 77.83\%\\
        & IWAE ($k=1$) & 227.44 & 13.35 & 8 & 77.61\%\\
        & IWAE ($k=5$) & 230.47 & 11.87 & 13 & 77.57\%\\
        & IWAE ($k=50$) & 246.58 & 10.04 & 13 & 77.78\%\\
        \hline
        \textbf{
        \parbox[t]{2.5mm}{\multirow{4}{*}{\rotatebox[origin=c]{90}{IGP}}}}
        & VAE & 229.92 & 20.14 & 32 & 82.39\%\\
        & IWAE ($k=1$) & 229.75 & 20.33 & 32 & 83.18\%\\
        & IWAE ($k=5$) & 246.22 & 14.75 & 32 & 82.37\%\\
        & IWAE ($k=50$) & 304.62 & 8.82 & 32 & 80.01\%\\
        \bottomrule
    \end{tabular}
    \caption{Reconstruction losses, KL, unit activation, and classification accuracy (CLS) of IGP and DGP models on Fashion-MNIST dataset. All results are calculated on test set.}
    \label{tab:image_loss}
\end{table}
\begin{figure*}[t]
    \centering
    \includegraphics[trim=0 40mm 0 40mm, clip,scale=0.15]{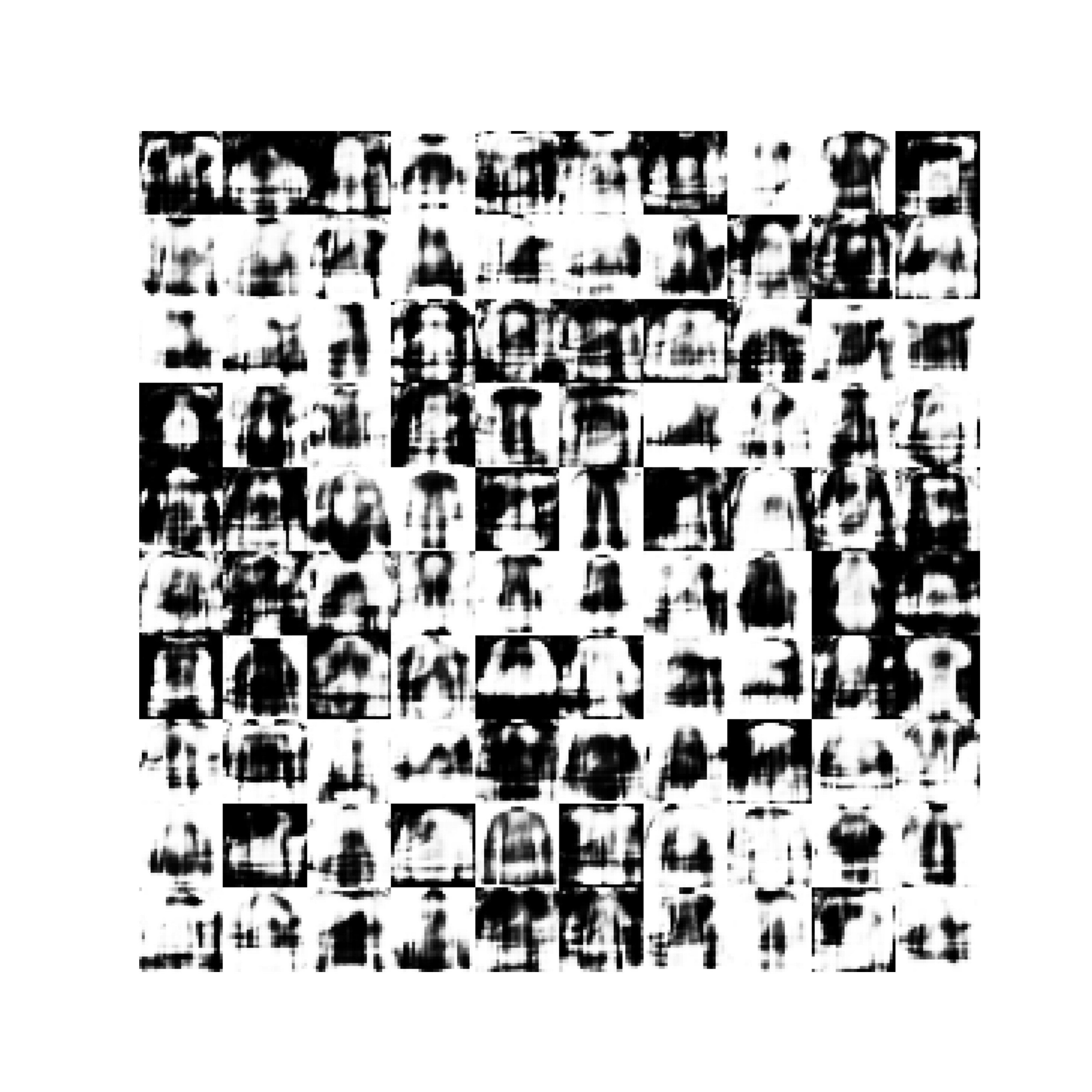}
    \includegraphics[trim=0 40mm 0 40mm, clip,scale=0.15]{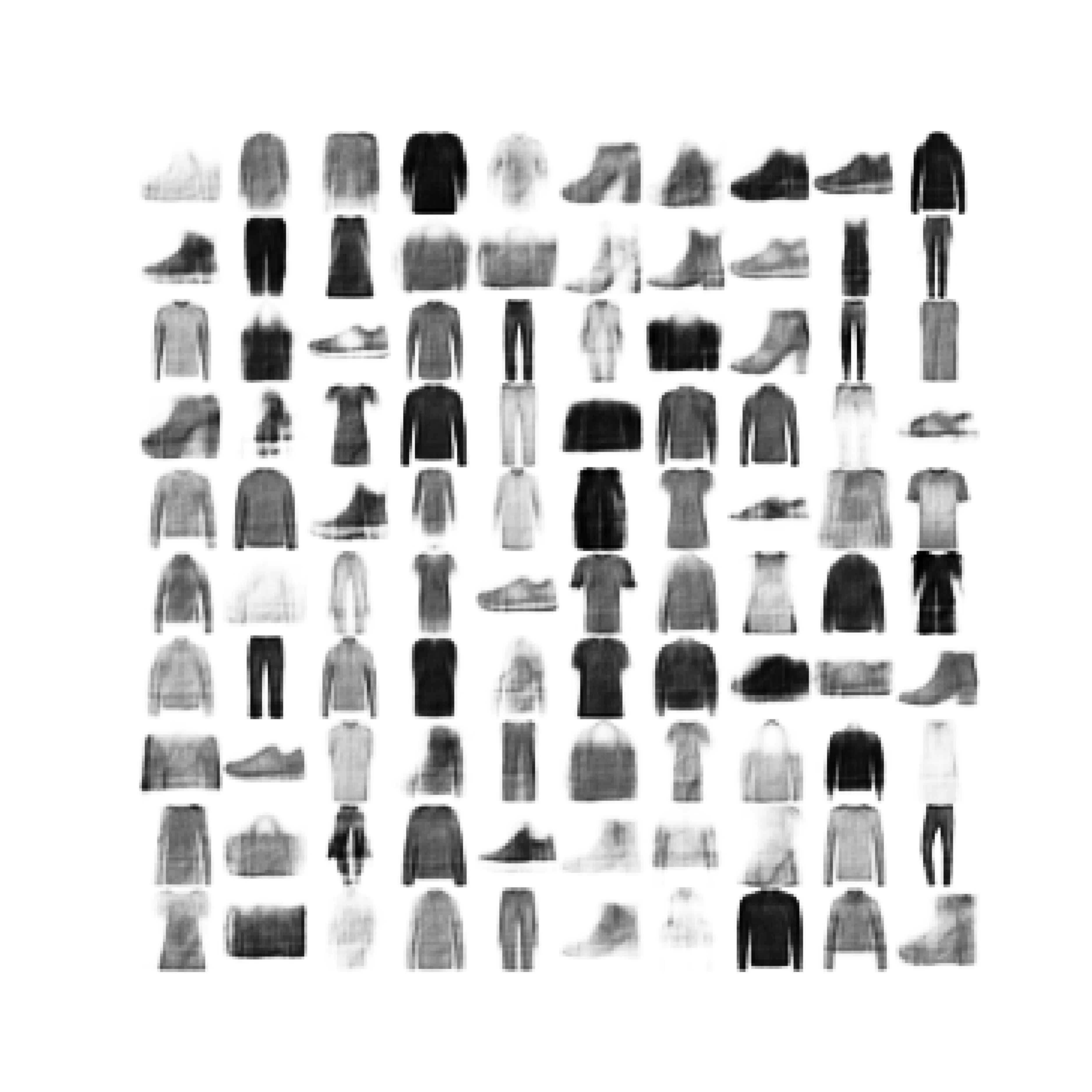}
    \includegraphics[trim=0 40mm 0 40mm, clip,scale=0.15]{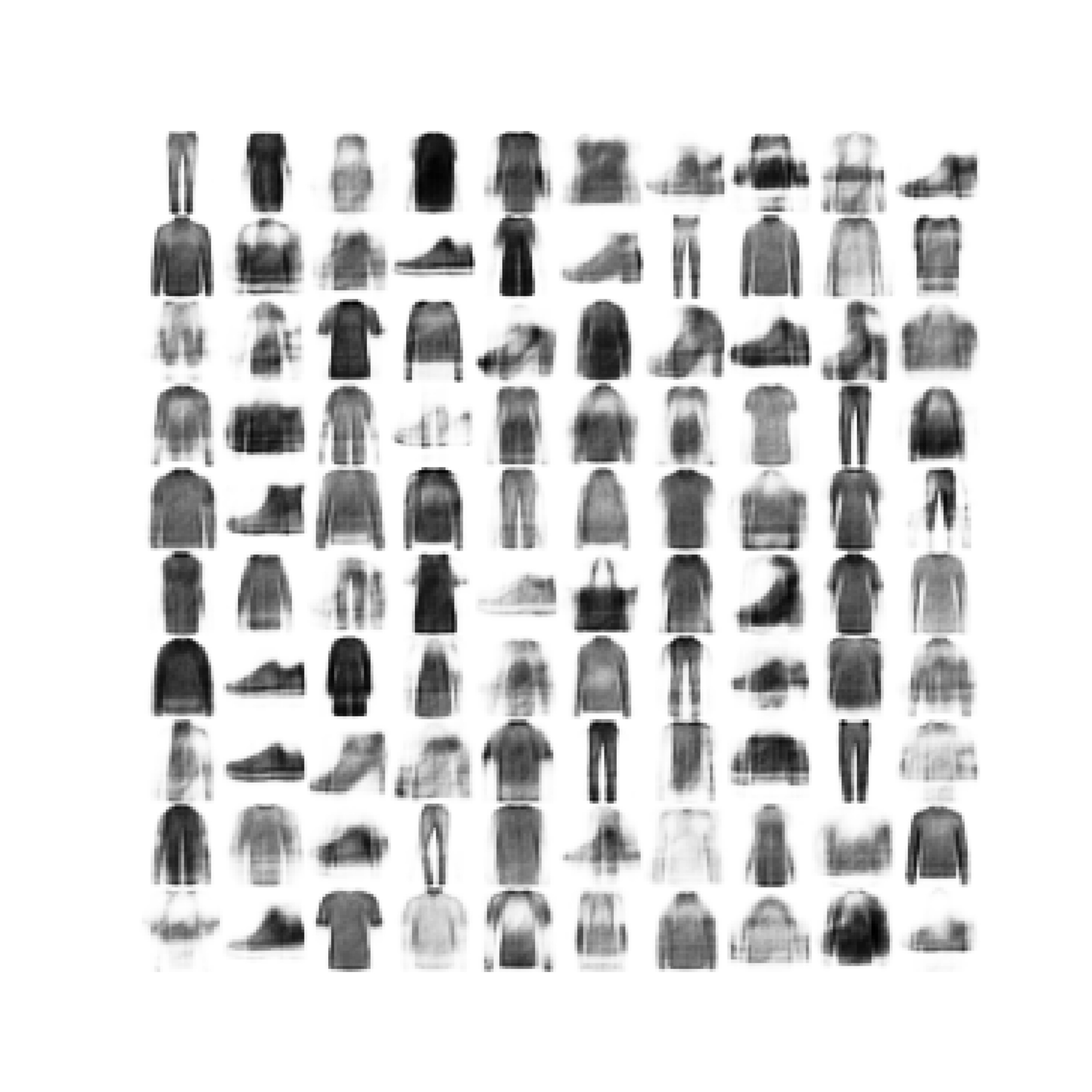}
    \caption{Generated images from random points sampled from prior. \textbf{Left:} AE; \textbf{Middle:} DGP, VAE; \textbf{Right:} IGP, VAE.}
    \label{fig:gen}
\end{figure*}

\subsection{Image Modality}
To examine the effectiveness of IGP, we run experiments on Fashion-MNIST \citep{fashion-mnist}.\footnote{See  \textit{Supplementary Material} for experiments on MNIST digits \citep{lecun2010mnist}.} Since KL-collapse is less of a concern in image domain (due to the absence of powerful auto-regressive decoding), we consider vanilla VAE with DGP and IGP posteriors, AE and IWAE \citep{burda2015} as our baselines. We follow the encoder-decoder architecture as in \citep{burda2015}, and set the latent dimension to 32. We train VAEs for 50 epochs with 128 batch size and Adam \citep{kingma2014} optimizer with learning rate 0.001, and trained with 5 different random starts. We select the model with smallest ELBO for VAE and IWAE, and lowest reconstruction loss for AE and report losses in Table~\ref{tab:image_loss}. Using IGP usually leads to a higher KL value and ELBO, but its effectiveness on activating all dimensions of the latent space is well-preserved on image domain.

\paragraph{Classification Experiment.} We use the latent code of images to do classification similar to text classification setting. Results are included in Table~\ref{tab:image_loss}. IGP VAEs still outperform DGP VAEs and AE, and achieves the best classification accuracy, whereas DGP VAEs cannot have a comparable accuracy with AE. This observation is consistent with our earlier observation for text.

\paragraph{Generative Experiment.} We use samples from prior to generate images in Figure~\ref{fig:gen}. As expected, AE cannot generate a valid image with clear boundary of object shape from random points, but both DGP and IGP models can generate images of various objects with different shapes. Intuitively, some images generated by DGP model have more emphasis on darker pixels and hence have a sharper boundary. In contrast, IGP model generates images with a lighter choice of pixels. However, there is no distinguishable difference between images generated by DGP and IGP VAE in terms of recognising objects.
\section{Analysis}\label{sec:analysis}
In this section we provide various theoretical and empirical analyses from different perspectives.

\subsection{Is This a Capacity Utilization Problem?}\label{sec:meanrep}
The encoder-decoder of VAEs can be seen as the sender-receiver in a communication network, where an input message $x$, is encoded into a compressed code $z$, and the receiver aims to decode $z$ back into $x$~\cite{alemi2018fixing}. The capacity of encoder in such a network can be measured in terms of the mutual information between $x$, $z$ bounded by:
$$H-D\leq I(x,z)\leq R,$$
where $R$ is ${\rm KL}(q_\phi(\mathbf{z}|\mathbf{x})||p(\mathbf{z}))$, and $D$ is $-\mathbb{E}_{q_\phi(\mathbf{z}|\mathbf{x})}[\log(p_\theta(\mathbf{x}|\mathbf{z}))]$, and H is the empirical data entropy (constant). In this framework, increasing the KL term can control the maximum capacity of the encoder channel. Tying variances of a diagonal Gaussian distribution can yield a higher KL-divergence requirement as well as learn a more representative mean vector for the latent space, which could be simply proved by the following theorem.

\begin{theorem}\label{the:rep}
Assume that diagonal Gaussian Distribution $p(\mathbf{x})=\mathcal{N}(\boldsymbol{\mu};diag(\sigma_1^2,\sigma_2^2,\dots,\sigma_n^2)$, is a probability density function for $n$ random variables $X_1,X_2,\dots,X_n$, and $q(\mathbf{y})=\mathcal{N}(\boldsymbol{\mu};\sigma^2\boldsymbol{I}_n)$, where $\sigma=\min(\sigma_1,\sigma_2,\cdots,\sigma_n)$, is another probability density function for $n$ random variables $Y_1,Y_2,\dots,Y_n$. Then, we have:
\begin{equation*}
    {\rm{KL}}(p(\mathbf{x})||\mathcal{N}(\boldsymbol{0};\boldsymbol{I}))\leq{\rm{KL}} (q(\mathbf{y})||\mathcal{N}(\boldsymbol{0};\boldsymbol{I})),
\end{equation*}
and for any $\varepsilon>0$,
\begin{equation*}
    Pr\!\left(\bigcap_{i=1}^{n}|X_i-\mu_i|<\varepsilon\!\right)\leq Pr\!\left(\bigcap_{i=1}^{n}|Y_i-\mu^{'}_i|<\varepsilon\!\right)
\end{equation*}
\end{theorem}
\begin{proof}
See \emph{Supplementary Material}.
\end{proof}

Hence, the direct use of isotropic Gaussian posterior can lead to a more optimal utilization of the latent space.
\begin{table}[t]
    \centering
    \begin{tabular}{l c c c c c }
        \toprule
        Config.& Rec. & KL & AU & BLEU-2/4\\
        \toprule
        $C=5$, DGP & 24.71 & 5.08 & 4 & 19.97/7.49\\
        $C=5$, IGP & 24.39 & 5.12 & 8 & 26.89/12.04\\
        \hline
        Base, IGP & 28.18 & 5.04 & 8 & 18.29/5.47\\
        Base, $C=5$, DGP & 24.53 & 5.11 & 6 & 24.48/10.60\\
        \bottomrule
    \end{tabular}
    \caption{Comparison of models trained fully under IGP or DGP vs hybrid models trained first via IGP, then initialised using this Base model and switched to DGP.}
    \label{tab:ablation}
\end{table}

\subsection{Is This an Optimization Problem?}
In the text domain, training VAEs faces the problem of posterior collapse. Popular strategies of avoiding posterior collapse in text such as $\beta$-VAE or the usage of C usually make the ELBO looser. It has been argued for that ELBO tightness might hurt the performance~\citep{pmlr-v80-rainforth18b}. Our empirical observation verifies the similar findings as tying the variances of a diagonal Gaussian posterior relaxes the tightness on ELBO and still leads to a better global optimum.

To highlight the optimisation advantage of IGP over DGP, we make an empirical observation on the SNLI dataset using batch size of 100 to first train a base VAE with IGP for 5 epochs with $C$ gradually increased to 5 (we refer to this as base model). Then using its weights, we untie variances to initialize a DGP model and further train it with $C=5$ from the base initialization. The results are provided in Table~\ref{tab:ablation}. 

We observe that using IGP as an initializer for vanilla VAE indeed improves performance in terms of reconstruction but does not make the DGP VAE to have the same AU under the same level of KL-divergence (compare rows $C=5$, IGP with Base, $C=5$, DGP). This indicates that the change of posterior might make the optimization of VAEs easier.

\subsection{Posterior Shape}
To understand the impact of isotropy on the aggregated posterior, $q_\phi(z)=\sum_{x\sim q(x)} q_\phi(z|x)$, we obtain unbiased samples of ${z}$ by sampling an ${x}$ from data and then ${z} \sim q_\phi({z}|{x})$, and measure the log determinant of covariance of the samples ($\log \det(\mathrm{Cov}[q_\phi({z})])$) as well as the mean of the samples to measure $||\mu||^2_2$. Table~\ref{tab:logdetcov} reports these for $C=15$. We observe that $\log \det(\mathrm{Cov}[q_\phi({z})])$ is significantly lower for IGP across indicating a sharper approximate posteriors. Additionally, the roughly matching $||\mu||^2_2$, in the presence of higher AU for IGP (see Figure~\ref{fig:loss}) suggests while both models were forced to transmit the same amount of information via a fixed channel capacity $C$, the IGP model has a more even distribution of information across different axes of the representation space, whereas DGP utilised only a small subspace. This, as observed earlier, translates into better discrimination in classification tasks.

\begin{table}[t]
    \centering
    \begin{tabular}{l c c c}
        \toprule
         & CBT & DBpedia & Yahoo\\
        \toprule
        DGP & $[0.11,-0.41]$ & $[0.11,-0.63]$ & $[0.10,-0.51]$\\
        IGP & $[0.11,-1.48]$ & $[0.12,-6.22]$ & $[0.08,-4.86]$\\
        \bottomrule
    \end{tabular}
    \caption{Reports $[\ ||\mu||^2_2\ ,\ \log \det(\mathrm{Cov}[q_\phi({z})])\ ]$.}
    \label{tab:logdetcov}
\end{table}
\section{Conclusion}\label{sec:conclusion}
In this work, we proposed a neglected variant of VAEs based on Isotropic Gaussian Posteriors (IGP) which addresses some of the sub-optimal properties of vanilla VAEs in representation space utilisation. We compare IGP VAEs with DGP VAEs, AEs, and IWAE on a series of quantitative and qualitative experiments including classification and generation. 

Our proposed modification consistently improved upon several established criteria from reconstruction quality to down stream task performance on various datasets. On the text classification tasks, IGP VAEs could outperform AEs, whereas DGP VAEs struggle to match AEs on performance. On the text generation tasks, IGP VAEs generally have better generation quality than DGP VAEs. These successes translate also into the image domain, specifically for downstream classification problem. Our empirical findings suggest IGP VAE is a good universal choice for both classification and generation tasks, as well as both text and image modalities.

Our ongoing work suggests the representation utilisation achieved by IGP has the potential to be exploited towards desired representational properties such as disentanglement. We are investigating such potentials as our future direction.
\bibliography{main.bib}
\clearpage

\appendix
\section{Proof of Theorem 1}
\begin{proof}
The first thing can be simply proved by the monotonic decreasing property of KL divergence in the Gaussian case. For the second thing, according to assumptions, we have:
\begin{small}
\begin{multline*}
    Pr\left(\bigcap_{i=1}^{n}|X_i-\mu_i|<\varepsilon\right)=
    \prod_{i=1}^{n}\int_{-\varepsilon}^{\varepsilon}\frac{1}{\sqrt{2\pi}\sigma_i}e^{-\frac{x_i^2}{2\sigma_i^2}}\,d{x_i},
\end{multline*}
\end{small}
and
\begin{small}
\begin{multline*}
    Pr\left(\bigcap_{i=1}^{n}|Y_i-\mu^{'}_i|<\varepsilon\right)=
    \prod_{i=1}^{n}\int_{-\varepsilon}^{\varepsilon}\frac{1}{\sqrt{2\pi}\sigma}e^{-\frac{y_i^2}{2\sigma^2}}\,d{y_i}\\=
    \left[\int_{-\varepsilon}^{\varepsilon}\frac{1}{\sqrt{2\pi}\sigma}e^{-\frac{y^2}{2\sigma^2}}\,d{y}\right]^n.
\end{multline*}
\end{small}
Construct auxiliary function:
\begin{small}
\begin{equation*}
    \begin{split}
    f(\sigma)&=\int_{-\varepsilon}^{\varepsilon}\frac{1}{\sqrt{2\pi}\sigma}e^{-\frac{x^2}{2\sigma^2}}\,dx=2\int_{0}^{\frac{\varepsilon}{\sqrt{2}\sigma}}\frac{1}{\sqrt{\pi}}e^{-x^2}\,dx.
    \end{split}
\end{equation*}
\end{small}
We have $f(\sigma)>0$. Calculating the first derivative of $f(\sigma)$, we have:
\begin{small}
\begin{equation*}
    f{'}(\sigma)=-\frac{2\varepsilon}{\sqrt{2\pi}\sigma^2}e^{-\frac{\varepsilon^2}{2\sigma^2}}<0.
\end{equation*}
\end{small}
Therefore, $f(\sigma)$ is a strictly decreasing function and $f(\sigma_i)\leq f(\sigma)$. We have:
\begin{small}
\begin{multline*}
    Pr\left(\bigcap_{i=1}^{n}|X_i-\mu_i|<\varepsilon\right)=\prod_{i=1}^{n}f(\sigma_i)\\\leq\\
    [f(\sigma)]^n=Pr\left(\bigcap_{i=1}^{n}|Y_i-\mu^{'}_i|<\varepsilon\right)
\end{multline*}
\end{small}
\end{proof}

\begin{table*}[t]
\centering
    \begin{tabular}{ll c c c c c}
        \toprule
        && Rec. & KL & AU & BLEU-2/4 & ROUGE-2/4\\
        \toprule
        \textbf{
        \parbox[t]{2mm}{\multirow{7}{*}{\rotatebox[origin=c]{90}{CBT}}}}
        & AE & $16.94_{\pm0.13}$ & - & $32.0_{\pm0.0}$ & $39.71_{\pm0.14}/30.81_{\pm0.12}$ & $43.51_{\pm0.11}/37.01_{\pm0.09}$\\\cdashlinelr{2-7}
        & $C=5$, DGP & $60.04_{\pm0.04}$ & $5.14_{\pm0.03}$ & $4.3_{\pm0.5}$ & $10.92_{\pm0.24}/2.17_{\pm0.07}$ & $5.04_{\pm0.24}/0.32_{\pm0.02}$\\
        & $C=5$, IGP & $60.50_{\pm0.02}$ & $5.09_{\pm0.03}$ & $32.0_{\pm0.0}$ & $15.45_{\pm0.18}/4.00_{\pm0.09}$ & $8.74_{\pm0.10}/0.78_{\pm0.03}$\\\cdashlinelr{2-7}
        & $C=15$, DGP & $52.11_{\pm0.11}$ & $15.11_{\pm0.04}$ & $11.0_{\pm1.6}$ & $14.76_{\pm0.04}/4.81_{\pm0.01}$ & $8.79_{\pm0.10}/2.13_{\pm0.02}$\\
        & $C=15$, IGP & $52.38_{\pm0.07}$ & $15.24_{\pm0.05}$ & $32.0_{\pm0.0}$ & $26.84_{\pm0.30}/11.73_{\pm0.15}$ & $20.54_{\pm0.21}/6.39_{\pm0.10}$\\\cdashlinelr{2-7}
        & $C=50$, DGP & $32.50_{\pm0.20}$ & $50.11_{\pm0.10}$ & $31.0_{\pm0.0}$ & $26.94_{\pm0.34}/16.53_{\pm0.20}$ & $25.39_{\pm0.39}/16.44_{\pm0.32}$\\
        & $C=50$, IGP & $31.97_{\pm0.08}$ & $50.31_{\pm0.07}$ & $32.0_{\pm0.0}$ & $33.27_{\pm0.30}/22.27_{\pm0.27}$ & $33.38_{\pm0.30}/23.65_{\pm0.22}$\\
        \hline
        \textbf{
          \parbox[t]{2mm}{\multirow{7}{*}{\rotatebox[origin=c]{90}{DBpedia}}} }
        & AE & $66.32_{\pm0.11}$ & - & $32.0_{\pm0.0}$ & $40.96_{\pm0.25}/27.57_{\pm0.17}$ & $35.87_{\pm0.19}/23.78_{\pm0.07}$\\\cdashlinelr{2-7}
        & $C=5$, DGP & $100.65_{\pm0.08}$ & $5.09_{\pm0.01}$ & $4.0_{\pm0.0}$ & $23.47_{\pm0.79}/14.00_{\pm0.47}$ & $20.85_{\pm0.69}/10.19_{\pm0.34}$\\
        & $C=5$, DGP & $101.73_{\pm0.31}$ & $5.04_{\pm0.01}$ & $32.0_{\pm0.0}$ & $25.09_{\pm0.55}/14.83_{\pm0.22}$ & $22.29_{\pm0.17}/10.47_{\pm0.10}$\\\cdashlinelr{2-7}
        & $C=15$, DGP & $94.16_{\pm0.19}$ & $15.06_{\pm0.04}$ & $8.7_{\pm0.9}$ & $35.35_{\pm0.49}/22.37_{\pm0.31}$ & $30.54_{\pm0.43}/17.41_{\pm0.16}$\\
        & $C=15$, IGP & $95.52_{\pm0.08}$ & $15.08_{\pm0.05}$ & $32.0_{\pm0.0}$ & $37.23_{\pm0.27}/24.47_{\pm0.11}$ & $34.19_{\pm0.12}/19.32_{\pm0.06}$\\\cdashlinelr{2-7}
        & $C=50$, DGP & $80.65_{\pm0.53}$ & $50.02_{\pm0.04}$ & $31.7_{\pm0.5}$ & $40.54_{\pm0.21}/26.95_{\pm0.19}$ & $35.19_{\pm0.25}/22.13_{\pm0.24}$\\
        & $C=50$, IGP & $80.58_{\pm0.04}$ & $50.15_{\pm0.04}$ & $32.0_{\pm0.0}$ & $44.79_{\pm0.30}/30.72_{\pm0.17}$ & $39.91_{\pm0.12}/25.40_{\pm0.08}$\\
        \hline
        \textbf{
        \parbox[t]{2mm}{\multirow{7}{*}{\rotatebox[origin=c]{90}{Yahoo Question}}}}
        & AE & $17.64_{\pm0.28}$ & - & $32.0_{\pm0.0}$ & $42.88_{\pm0.51}/32.86_{\pm0.58}$ & $41.63_{\pm0.58}/31.67_{\pm0.67}$\\\cdashlinelr{2-7}
        & $C=5$, DGP & $50.58_{\pm0.06}$ & $5.14_{\pm0.01}$ & $5.7_{\pm0.5}$ & $17.07_{\pm0.71}/6.04_{\pm0.25}$ & $10.96_{\pm0.41}/1.50_{\pm0.06}$\\
        & $C=5$, IGP & $51.24_{\pm0.01}$ & $5.06_{\pm0.03}$ & $32.0_{\pm0.0}$ & $20.91_{\pm0.03}/8.07_{\pm0.03}$ & $14.48_{\pm0.13}/2.21_{\pm0.02}$\\\cdashlinelr{2-7}
        & $C=15$, DGP & $43.00_{\pm0.12}$ & $15.06_{\pm0.04}$ & $9.3_{\pm1.2}$ & $22.62_{\pm0.37}/10.81_{\pm0.21}$ & $16.04_{\pm0.32}/4.76_{\pm0.08}$\\
        & $C=15$, IGP & $44.43_{\pm0.05}$ & $15.20_{\pm0.12}$ & $32.0_{\pm0.0}$ & $29.76_{\pm0.06}/14.99_{\pm0.08}$ & $23.11_{\pm0.17}/6.94_{\pm0.12}$\\\cdashlinelr{2-7}
        & $C=50$, DGP & $28.29_{\pm0.40}$ & $50.00_{\pm0.19}$ & $31.3_{\pm0.9}$ & $31.78_{\pm0.73}/20.47_{\pm0.70}$ & $27.14_{\pm0.85}/15.07_{\pm0.77}$\\
        & $C=50$, IGP & $26.18_{\pm0.19}$ & $50.15_{\pm0.08}$ & $32.0_{\pm0.0}$ & $39.68_{\pm0.20}/27.49_{\pm0.31}$ & $35.73_{\pm0.40}/22.57_{\pm0.55}$\\
        \hline
        \textbf{
        \parbox[t]{2mm}{\multirow{6}{*}{\rotatebox[origin=c]{90}{SNLI}}}}
        & $C=5$, DGP & $24.70_{\pm0.08}$ & $5.10_{\pm0.02}$ & $4.3_{\pm0.5}$ & $21.84_{\pm1.57}/8.74_{\pm0.90}$ & $15.89_{\pm1.47}/2.52_{\pm0.37}$\\
        & $C=5$, IGP & $24.54_{\pm0.18}$ & $5.12_{\pm0.01}$ & $8.0_{\pm0.0}$ & $26.59_{\pm0.82}/12.00_{\pm0.97}$ & $20.57_{\pm0.86}/4.41_{\pm0.89}$\\\cdashlinelr{2-7}
        & $C=15$, DGP & $19.25_{\pm0.42}$ & $15.04_{\pm0.04}$ & $7.0_{\pm0.0}$ & $31.34_{\pm0.24}/18.99_{\pm0.52}$ & $28.07_{\pm0.39}/12.79_{\pm0.77}$ \\
        & $C=15$, IGP & $18.58_{\pm0.18}$ & $15.09_{\pm0.01}$ & $8.0_{\pm0.0}$ & $35.56_{\pm0.45}/22.94_{\pm0.49}$ & $32.75_{\pm0.63}/16.88_{\pm0.60}$\\\cdashlinelr{2-7}
        & $C=50$, DGP & $20.35_{\pm0.63}$ & $49.94_{\pm0.04}$ & $8.0_{\pm0.0}$ & $26.16_{\pm1.11}/14.15_{\pm1.08}$ & $21.78_{\pm1.25}/8.14_{\pm1.06}$\\
        & $C=50$, IGP & $18.43_{\pm0.71}$ & $50.00_{\pm0.07}$ & $8.0_{\pm0.0}$ & $30.78_{\pm1.60}/18.81_{\pm1.69}$ & $27.50_{\pm1.77}/12.95_{\pm1.72}$\\
        \bottomrule
    \end{tabular}
    \caption{Results are calculated on the test set. We report mean value and standard deviation across 3 runs. Rec, AU, and PPL denote reconstruction loss, number of Active Units and estimated perplexity, respectively. DGP, and IGP denote diagonal Gaussian posteriors and isotropic Gaussian posteriors, respectively. $C$ is the target KL value.}
    \label{tab:losss}
\end{table*}


\begin{table*}[t]
    \centering
    \begin{tabular}{ll c c c c c}
        \toprule
        && Rec. & KL & AU & BLEU-2/4 & ROUGE-2/4\\
        \toprule
        \textbf{
         \parbox[t]{2mm}{\multirow{6}{*}{\rotatebox[origin=c]{90}{CBT}}}}
        & IG/DGP & $60.04_{\pm0.04}$ & $5.14_{\pm0.03}$ & $4.3_{\pm0.5}$ & $10.92_{\pm0.24}/2.17_{\pm0.07}$ & $5.04_{\pm0.24}/0.32_{\pm0.02}$\\
        & IG/IGP & $60.50_{\pm0.02}$ & $5.09_{\pm0.03}$ & $32.0_{\pm0.0}$ & $15.45_{\pm0.18}/4.00_{\pm0.09}$ & $8.74_{\pm0.10}/0.78_{\pm0.03}$\\\cdashlinelr{2-7}
        & MoG/DGP & $61.42_{\pm0.40}$ & $4.99_{\pm0.11}$ & $1.0_{\pm0.0}$ & $6.58_{\pm0.29}/1.04_{\pm0.06}$ & $2.35_{\pm0.02}/0.10_{\pm0.01}$\\
        & MoG/IGP & $60.57_{\pm0.45}$ & $6.03_{\pm0.06}$ & $20.3_{\pm13.0}$ & $8.79_{\pm0.36}/1.53_{\pm0.13}$ & $3.48_{\pm0.36}/0.21_{\pm0.04}$\\\cdashlinelr{2-7}
        & MoIG/DGP & $61.42_{\pm0.20}$ & $4.98_{\pm0.05}$ & $1.0_{\pm0.0}$ & $6.73_{\pm0.05}/1.00_{\pm0.03}$ & $2.51_{\pm0.03}/0.09_{\pm0.02}$\\
        & MoIG/IGP & $60.46_{\pm0.14}$ & $6.09_{\pm0.10}$ & $32.0_{\pm0.0}$ & $13.32_{\pm0.19}/2.91_{\pm0.05}$ & $6.84_{\pm0.17}/0.48_{\pm0.01}$\\
        \hline
        \textbf{
        \parbox[t]{2mm}{\multirow{6}{*}{\rotatebox[origin=c]{90}{SNLI}}}}
        & IG/DGP & $24.70_{\pm0.08}$ & $5.10_{\pm0.02}$ & $4.3_{\pm0.5}$ & $21.84_{\pm1.57}/8.74_{\pm0.90}$ & $15.89_{\pm1.47}/2.52_{\pm0.37}$\\
        & IG/IGP & $24.54_{\pm0.18}$ & $5.12_{\pm0.01}$ & $8.0_{\pm0.0}$ & $26.59_{\pm0.82}/12.00_{\pm0.97}$ & $20.57_{\pm0.86}/4.41_{\pm0.89}$\\\cdashlinelr{2-7}
        & MoG/DGP & $25.76_{\pm0.04}$ & $4.84_{\pm0.01}$ & $1.0_{\pm0.0}$ & $13.45_{\pm0.10}/4.60_{\pm0.09}$ & $8.14_{\pm0.38}/1.06_{\pm0.07}$\\
        & MoG/IGP & $28.40_{\pm0.23}$ & $4.78_{\pm0.04}$ & $1.0_{\pm0.0}$ & $9.89_{\pm1.36}/2.32_{\pm0.52}$ & $5.20_{\pm0.83}/0.23_{\pm0.08}$\\\cdashlinelr{2-7}
        & MoIG/DGP & $25.81_{\pm0.02}$ & $4.93_{\pm0.03}$ & $1.3_{\pm0.5}$ & $13.45_{\pm0.09}/4.40_{\pm0.10}$ & $7.52_{\pm0.07}/0.88_{\pm0.03}$\\
        & MoIG/IGP & $23.51_{\pm0.04}$ & $6.35_{\pm0.01}$ & $8.0_{\pm0.0}$ & $27.42_{\pm0.27}/12.64_{\pm0.22}$ & $20.83_{\pm0.48}/4.98_{\pm0.25}$\\
        \bottomrule
    \end{tabular}
    \caption{Results are calculated on the test set. We report mean value and standard deviation across 3 runs. IG, MoG, and MoIG denote isotropic Gaussian prior, Mixture of Gaussian prior, and Mixture of isotropic Gaussian prior, respectively.}
    \label{tab:prior}
\end{table*}
\section{Full Results}\label{app:alll}
We report detailed reconstruction loss, KL-divergence, active units and results of BLEU and ROUGE scores on reconstructed test set in Table~\ref{tab:losss}.
\section{Results on Mixture Priors}
In addition to standard Isotropic Gaussian prior, we also test how well both posteriors work with Mixture of Gaussian (MoG) prior \citep{pelsmaeker-aziz-2020-effective}. 
Since the regular MoG prior incorporates several diagonal Gaussian distributions, we also employ Mixture of Isotropic Gaussian (MoIG) to see the influence of introducing isotropy. We set the number of components of priors to 5 and train VAEs on CBT. Table~\ref{tab:prior} reports reconstruction loss, BLEU-2 and active units. 

As observed using either MoG or MoIG as prior degrades VAEs performance for both posteriors. The combination of MoIG prior and Isotropic posterior has a clear advantage of MoG due to consistency between prior and posterior distributional families. In all cases IGP outperforms its vanilla VAE counterpart (DGP), indicating  the better compatibility of IGP with other priors.

\begin{table}[t]
    \centering
    \begin{tabular}{ll c c}
        \toprule
         && \textbf{F1} & \textbf{Accuracy}\\
        \hline
        \textbf{
          \parbox[t]{3mm}{\multirow{8}{*}{\rotatebox[origin=c]{90}{DBpedia}}}}
        & AE & $90.08_{\pm0.66}$ & $90.63_{\pm0.67}$\\\cdashlinelr{2-4}
        & $C=5$, DGP & $52.76_{\pm3.76}$ & $81.88_{\pm3.11}$\\
        & $C=5$, IGP & $45.72_{\pm0.50}$ & $80.62_{\pm0.97}$\\\cdashlinelr{2-4}
        & $C=15$, DGP & $78.91_{\pm1.68}$ & $87.10_{\pm0.57}$\\
        & $C=15$, IGP & $82.62_{\pm0.73}$ & $91.11_{\pm0.41}$\\\cdashlinelr{2-4}
        & $C=50$, DGP & $87.68_{\pm0.36}$ & $88.74_{\pm0.23}$\\
        & $C=50$, IGP & $91.69_{\pm0.36}$ & $91.00_{\pm0.45}$\\\cdashlinelr{2-4}
        & Original & 98.08 & -\\
        \hline
        \textbf{
        \parbox[t]{3mm}{\multirow{8}{*}{\rotatebox[origin=c]{90}{Yahoo Question}}}}
        & AE & $37.28_{\pm0.34}$ & $35.58_{\pm0.60}$\\\cdashlinelr{2-4}
        & $C=5$, DGP & $14.00_{\pm0.41}$ & $20.75_{\pm0.13}$\\
        & $C=5$, IGP & $16.55_{\pm0.27}$ & $25.38_{\pm0.73}$\\\cdashlinelr{2-4}
        & $C=15$, DGP & $20.66_{\pm0.55}$ & $28.39_{\pm0.32}$\\
        & $C=15$, IGP & $31.32_{\pm0.12}$ & $40.03_{\pm0.49}$\\\cdashlinelr{2-4}
        & $C=50$, DGP & $29.40_{\pm0.19}$ & $33.05_{\pm1.02}$\\
        & $C=50$, IGP & $37.34_{\pm0.21}$ & $36.92_{\pm0.13}$\\\cdashlinelr{2-4}
        & Original & 57.55 & -\\
        \bottomrule
    \end{tabular}
    \caption{Macro-averaged \textbf{F1}-score of agreement on reconstructed test set and classification accuracy of representations from VAE encoder.}
    \label{tab:aggreement}
\end{table}

\section{Full Agreement and Classification Results}
The results of agreement experiment based on reconstructed test set and classification task based on the representations from VAEs trained on classification datasets are shown in Table~\ref{tab:aggreement} for three runs.

\section{Results on MNIST}
We show the losses of models trained on MNIST, and their classification accuracy in Table~\ref{tab:image_losss}. We also provide generation examples of AE, DGP VAE and IGP VAE in Figure~\ref{fig:clff}.
\begin{table}[t]
\centering
    \begin{tabular}{l l *{4}{c}}
        \toprule\\
        & Model & Rec. & KL & AU & CLF\\
        \hline
        & AE & 58.86 & - & 32 & 90.82\%\\
        \hline
        \textbf{
        \parbox[t]{2.5mm}{\multirow{4}{*}{\rotatebox[origin=c]{90}{DGP}}}}
        & VAE & 78.64 & 21.98 & 15 & 91.52\%\\
        & IWAE ($k=1$) & 79.29 & 21.65 & 15 & 91.49\%\\
        & IWAE ($k=5$) & 81.62 & 20.42 & 18 & 91.74\%\\
        & IWAE ($k=50$) & 91.03 & 17.74 & 21 & 91.91\%\\
        \hline
        \textbf{
        \parbox[t]{2.5mm}{\multirow{4}{*}{\rotatebox[origin=c]{90}{IGP}}}} 
        & VAE & 78.93 & 29.09 & 32 & 88.82\%\\
        & IWAE ($k=1$) & 79.51 & 28.64 & 32 & 88.87\%\\
        & IWAE ($k=5$) & 86.61 & 24.42 & 32 & 89.81\%\\
        & IWAE ($k=50$) & 104.14 & 19.47 & 32 & 90.71\%\\
        \bottomrule
    \end{tabular}
    \caption{Comparison of IGP and DGP models on MNIST datasets.}
    \label{tab:image_losss}
\end{table}

\begin{figure*}[t]
    \centering
    \includegraphics[scale=0.15]{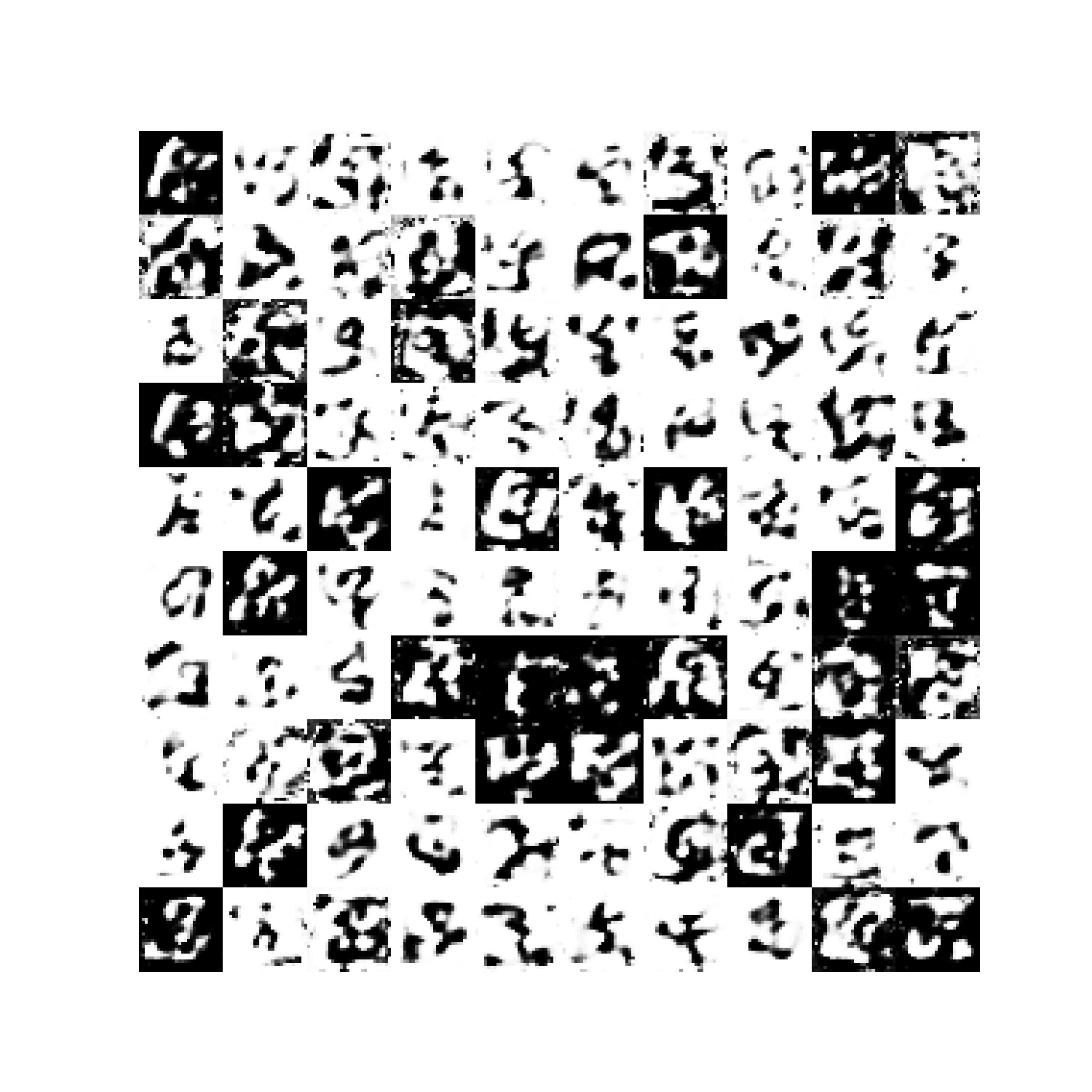}
    \includegraphics[scale=0.15]{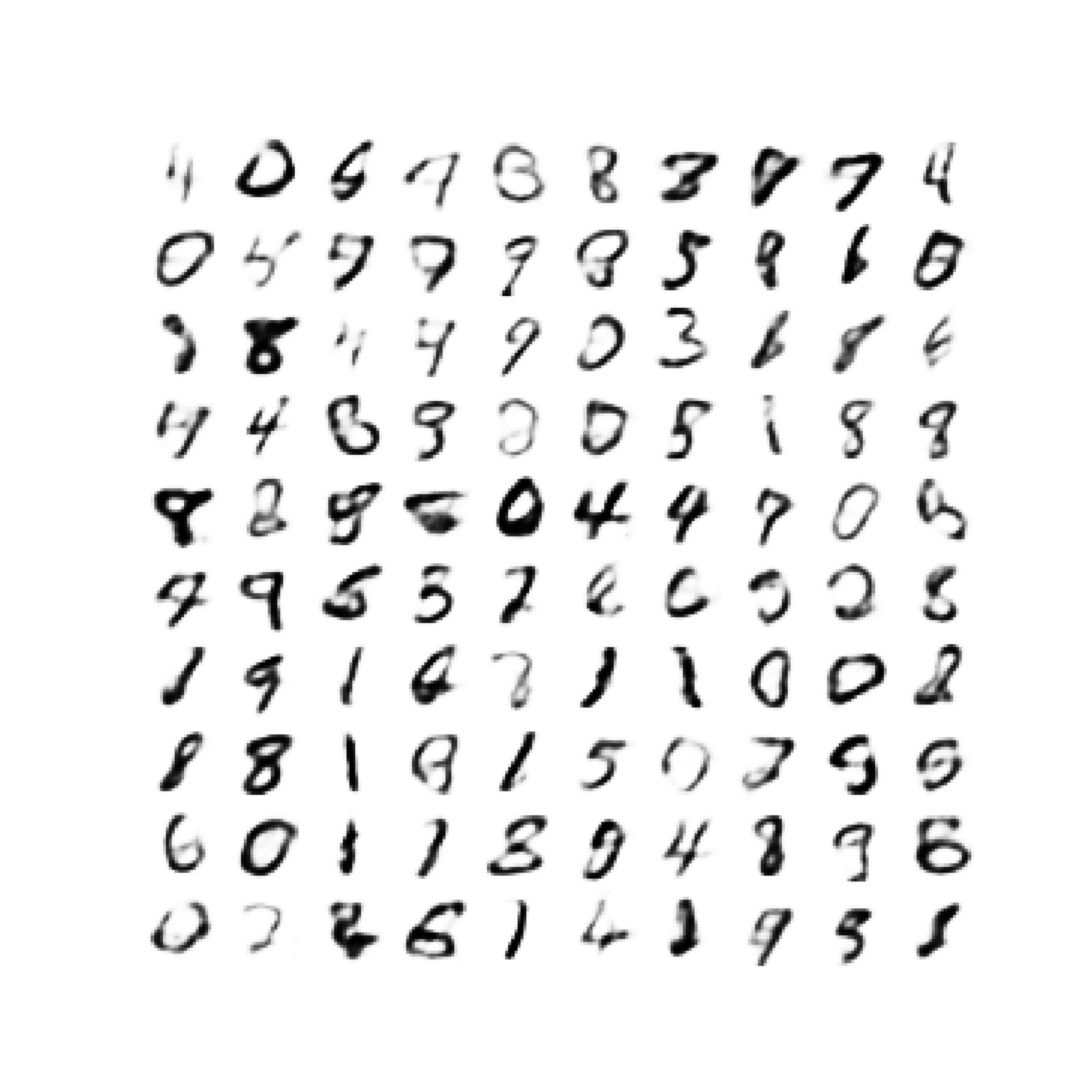}
    \includegraphics[scale=0.15]{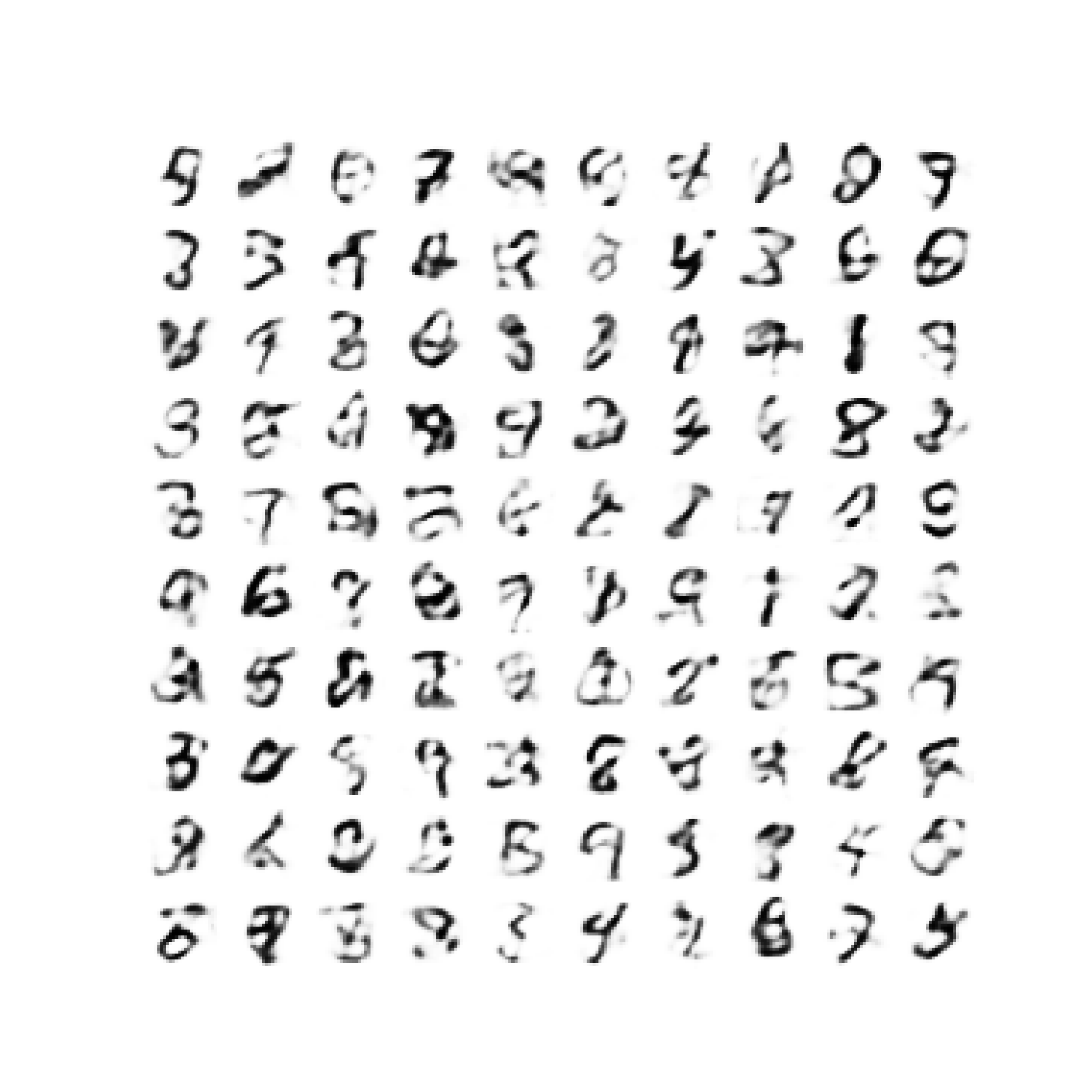}
    \caption{\textbf{Left:} AE; \textbf{Middle:} DGP, VAE; \textbf{Right:} IGP, VAE ($k=50$).}
    \label{fig:clff}
\end{figure*}
\end{document}